\documentclass[12pt, pdftex]{article} 

\usepackage{times}
\usepackage{fullpage}
\usepackage{url}

\usepackage{graphicx}
\graphicspath{{./image/}}

\usepackage{hyperref}
\hypersetup{colorlinks=true, citecolor=blue, anchorcolor=blue, linkcolor=blue,
    filecolor=blue, menucolor=blue, urlcolor=blue, plainpages=false,
    pdfpagelabels}

\usepackage[numbers]{natbib}
\usepackage{enumerate}
\usepackage{amsmath}
\usepackage{amsfonts}
\usepackage{amssymb}
\usepackage{amsthm}
\usepackage{mathtools}
\usepackage{subfig}
\usepackage{booktabs}
\usepackage[vlined, ruled]{algorithm2e}

\def \P{{\mathbb P}}
\def \epsilon{\varepsilon}
\let\tilde\widetilde
\let\hat\widehat

\newcommand{\abs}[1]{|#1|}

\newcommand{\norm}[1]{\left\|#1\right\|}
\newcommand{\wt}[1]{\widetilde{#1}}
\newcommand{\wh}[1]{\widehat{#1}}
\newcommand{\set}[1]{\left\{#1\right\}}
\newcommand{\T}{\top}
\def \vec{\text{vec}}
\DeclarePairedDelimiter\ip{\langle}{\rangle}

\newcommand{\C}{\mathbb{C}}
\newcommand{\R}{\mathbb{R}}
\newcommand{\E}{\mathbb{E}}
\newcommand{\N}{\mathcal{N}}
\newcommand{\A}{\mathcal{A}}
\newcommand{\tS}{\mathcal{S}}
\newcommand{\bS}{\mathbb{S}}
\newcommand{\Var}{\mathbb{V}\text{ar}}
\renewcommand{\r}{r}

\newcommand{\mA}{\boldsymbol{A}}

\newcommand{\overbar}[1]{\mkern 1.5mu\overline{\mkern-2.5mu#1\mkern-1mu}\mkern 1.5mu}
\newcommand{\Zbar}{\overbar{Z}}
\newcommand{\zbar}{\bar{z}}
\newcommand{\Ztilde}{\tilde{Z}}
\newcommand{\ztilde}{\tilde{z}}

\newcommand{\Zstar}{Z^\star}
\newcommand{\Xstar}{X^\star}
\newcommand{\zstar}{z^\star}
\newcommand{\Ubar}{\bar{U}}

\DeclareMathOperator{\diag}{\text{diag}}
\DeclareMathOperator{\rank}{\text{rank}}
\DeclareMathOperator{\trace}{\text{tr}}
\DeclareMathOperator*{\argmin}{arg\,min}
\DeclareMathOperator*{\argmax}{arg\,max}

\newtheorem{theorem}{Theorem}
\newtheorem{lemma}{Lemma}
\newtheorem{coro}{Corollary}
\newtheorem{defn}{Definition}

\title{\Large\bf A Convergent Gradient Descent Algorithm for \\ Rank Minimization
and Semidefinite Programming \\ from Random Linear Measurements}

\author{
\normalsize Qinqing Zheng \quad John Lafferty\\
\normalsize University of Chicago \\
}
\date{}

\begin{document}
\maketitle

\begin{abstract}
We propose a simple, scalable, and fast gradient descent algorithm to
optimize a nonconvex objective for the rank minimization problem and a
closely related family of semidefinite programs.  With $O(\r^3
\kappa^2 n \log n)$ random measurements of a positive semidefinite
$n\times n$ matrix of rank $\r$ and condition number $\kappa$, our
method is guaranteed to converge linearly to the global optimum.
\end{abstract}

\section{Introduction}
Semidefinite programming has become a key optimization tool in many
areas of applied mathematics, signal processing and machine learning.
SDPs often arise naturally from the problem structure, or are derived
as surrogate optimizations that are relaxations of difficult
combinatorial problems \citep{Aspremont:04,amini:09,Goemans:1995}.  In
spite of the importance of SDPs in principle---promising efficient
algorithms with polynomial runtime guarantees---it is widely recognized
that current optimization algorithms based on interior point methods
can handle only relatively small problems.  Thus, a considerable gap
exists between the theory and applicability of SDP formulations.
Scalable algorithms for semidefinite programming, and closely related
families of nonconvex programs more generally, are greatly needed.

A parallel development is the surprising effectiveness of simple
classical procedures such as gradient descent for large
scale problems, as explored in the recent machine learning literature.
In many areas of machine learning and signal processing such as
classification, deep learning, and phase retrieval, gradient descent
methods, in particular first order stochastic optimization, have led
to remarkably efficient algorithms that can attack very large scale
problems \cite{Bach:11,bach:14,hoffman:13,CanLiSol14}. In this
paper we build on this work to develop first-order algorithms for
solving the rank minimization problem under random measurements and
a closely related family of semidefinite programs. Our algorithms
are efficient and scalable, and we prove that they attain linear
convergence to the global optimum under natural assumptions.

The affine rank minimization problem is to find a matrix 
$X^\star \in \R^{n\times p}$ of minimum rank satisfying constraints $\A(X^\star) = b$,
where $\A: \R^{n \times p} \longrightarrow \R^m$ is an affine transformation.
The underdetermined case where $m \ll np$ is of particular interest, and
can be formulated as the optimization
\begin{equation}
    \label{eq:affine_rank_min}
\begin{aligned}
    & \min_{X \in \R^{n \times p} } && \rank(X)\\
    & \text{subject to} && \A(X) = b. 
\end{aligned}
\end{equation}
This problem is a direct generalization of compressed sensing,
and subsumes many machine learning problems such as image compression, low rank matrix completion and low-dimensional metric
embedding
\cite{RecFazPar10, JaiNetSan13}. While the problem is natural
and has many applications, the optimization is nonconvex and challenging to solve.
Without conditions on the transformation $\A$ or the minimum rank
solution $X^\star$, it is generally NP hard \cite{MekJaiCar08}. 

Existing methods, such as nuclear norm relaxation \cite{RecFazPar10},
singular value projection (\texttt{SVP}) \cite{JaiMekDhi10}, and alternating least
squares (\texttt{AltMinSense}) \cite{JaiNetSan13}, assume that a certain
restricted isometry property (RIP) holds for $\A$. In the random measurement
setting, this essentially means that at least $O(\r (n+p) \log(n+p) )$
measurements are available, where $\r = \rank(\Xstar)$ \cite{RecFazPar10}.
In this work, we assume that
(i) $X^\star$ is positive semidefinite and
(ii) $\A: \R^{n \times n} \longrightarrow \R^m$ is defined as $\A(X)_i = \trace(A_i X)$, where
each $A_i$ is a random $n \times n$ symmetric matrix from the Gaussian Orthogonal
Ensemble (GOE), with
$(A_i)_{jj} \sim \N(0,2)$ and 
$(A_i)_{jk} \sim \N(0,1)$ for $j \neq k$. 
Our goal is thus to solve the optimization
\begin{equation}
    \label{eq:affine_rank_psd}
\begin{aligned}
    &\min_{X \succeq 0 } &&\rank(X)\\
    &\text{subject to} && \trace(A_i X) = b_i, \;\; i = 1, \ldots, m. \\
\end{aligned}
\end{equation}
In addition to the wide applicability of affine rank minimization,
the problem is also closely connected to
a class of semidefinite programs.
In Section \ref{sec:conn_sdp}, we show that the minimizer
of a particular class of SDP can be obtained by a linear transformation
of $\Xstar$. Thus, efficient algorithms for problem \eqref{eq:affine_rank_psd}
can be applied in this setting as well.

Noting that a
rank-$\r$ solution $\Xstar$ 
to \eqref{eq:affine_rank_psd} 
can be decomposed as $\Xstar = \Zstar {\Zstar}^\T$
where $\Zstar \in
\R^{n \times \r}$, our approach is based on minimizing
the squared residual
\[ 
    f(Z) = \frac{1}{4m}\norm{\A(ZZ^\T) - b}^2 = \frac{1}{4m} \sum_{i=1}^m
    \left(\trace(Z^\T A_i Z) - b_i \right)^2.
\]
While this is a nonconvex function, we take motivation from recent work for phase retrieval
by \citet{CanLiSol14}, and develop a gradient descent algorithm for
optimizing $f(Z)$, using a carefully constructed initialization and step size. 
Our main contributions concerning this algorithm are as follows.
\begin{itemize}
    \item We prove that with $O(\r^3 n\log n)$ constraints our
      gradient descent scheme can  
      exactly recover $\Xstar$ with
      high probability. Empirical experiments show that this bound
      may potentially be improved to $O(\r n \log n)$.

    \item We show that our method converges linearly, and 
     has lower computational cost compared with previous methods.
        
    \item We carry out a detailed comparison of rank minimization
      algorithms, and demonstrate that when the measurement matrices
      $A_i$ are sparse, our gradient method significantly
      outperforms alternative approaches.
\end{itemize}
In Section \ref{sec:related} we briefly
review related work. In Section \ref{sec:algo} we discuss the gradient
scheme in detail. Our main analytical results are presented in
Section~\ref{sec:rst},
with detailed proofs contained in the supplementary material.
Our experimental results are presented in Section \ref{sec:expr}, 
and we conclude with a brief discussion of future work in Section
\ref{sec:conclude}.

%
%
%

\section{Semidefinite Programming and Rank Minimization}
\label{sec:conn_sdp}
Before reviewing related work and presenting our algorithm, we pause
to explain the connection between semidefinite programming and rank
minimization. This connection enables our scalable gradient descent
algorithm to be applied and analyzed for certain classes of SDPs.

Consider a standard form semidefinite program
\begin{equation}
    \label{eq:sdp_C_psd}
\begin{aligned}
    & \min_{\tilde{X} \succeq 0}  && \trace(\tilde{C} \tilde{X})\\
    & \text{subject to} && \trace(\tilde{A}_i \tilde{X}) = b_i, \;\; i = 1, \ldots, m \\
\end{aligned}
\end{equation}
where $\tilde{C}, \tilde{A}_1, \ldots, \tilde{A}_m \in \bS^n$.
If $\tilde{C}$ is positive definite, then we can write
$\tilde{C} = LL^\T$ where $L \in \R^{n \times n}$
is invertible. It follows that the minimum of problem \eqref{eq:sdp_C_psd} is
the same as
\begin{equation}
    \label{eq:sdp_no_C}
\begin{aligned}
    & \min_{X \succeq 0} && \trace(X)\\
    & \text{subject to} && \trace(A_i X) = b_i, \;\; i = 1, \ldots, m \\
\end{aligned}
\end{equation}
where $A_i = L^{-1} \tilde{A}_i {L^{-1}}^\top$. In particular, minimizers $\tilde{X}^*$ of
\eqref{eq:sdp_C_psd} are obtained from minimizers $X^*$ 
of \eqref{eq:sdp_no_C} via the transformation
\[ \tilde{X}^* =  {L^{-1}}^\top X^* L^{-1}. \] 
Since $X$ is positive semidefinite, $\trace(X)$ is equal to $\norm{X}_*$.
Hence, problem \eqref{eq:sdp_no_C} is the nuclear norm relaxation of problem
\eqref{eq:affine_rank_psd}. 
Next, we characterize the specific cases where $X^* = X^\star$, so
that the SDP and rank minimization solutions coincide.  The following 
result is from \citet{RecFazPar10}.
\begin{theorem}
Let $\A: \R^{n \times n} \longrightarrow \R^m$ be a linear map. For every
integer $k$ with $1\leq k \leq n$, define the $k$-restricted isometry
constant to be the smallest value $\delta_k$ such that 
\[ (1-\delta_k) \norm{X}_F \leq \norm{\A(X)} \leq (1+\delta_k) \norm{X}_F  \]
holds for any matrix $X$ of rank at most $k$.
Suppose that there
exists a rank $r$ matrix $X^\star$ such that $\A(X^\star) = b$. If
$\delta_{2r} < 1$, then $X^\star$ is the only matrix of rank at most $r$ satisfying $\A(X)
= b$. Furthermore, if $\delta_{5r} < 1 / 10$, then $X^\star$ can be attained by
minimizing $\norm{X}_*$ over the affine subset.
\end{theorem}
In other words, since $\delta_{2r} \leq \delta_{5r}$,
if $\delta_{5r} < 1 / 10$ holds for the transformation $\A$ and one
finds a matrix $X$ of rank $\r$ satisfying the affine constraint, then
$X$ must be positive semidefinite. Hence, one can ignore the
semidefinite constraint $X \succeq 0$ when
solving the rank minimization \eqref{eq:affine_rank_psd}. The resulting problem
then can be exactly solved by nuclear norm relaxation. 
Since the minimum rank solution is positive semidefinite, 
it then coincides with the solution of the SDP \eqref{eq:sdp_no_C}, which is a constrained nuclear norm
optimization.

The observation that one can ignore the semidefinite constraint 
justifies our experimental comparison with methods such as 
nuclear norm relaxation, \texttt{SVP}, and \texttt{AltMinSense},
described in the following section.

\section{Related Work}
\label{sec:related}
Burer and Monteiro \cite{BurMon03}
proposed a general approach for solving semidefinite programs 
using factored, nonconvex optimization, giving mostly experimental 
support for the convergence of the algorithms.
The first nontrivial guarantee for solving affine rank minimization problem 
is given by 
\citet{RecFazPar10}, based on replacing the rank function by the
convex surrogate nuclear norm, as already mentioned in the previous section.
While this is a convex
problem, solving it in practice is nontrivial, and a variety of methods have been
developed for efficient nuclear norm minimization.  The most popular algorithms
are proximal methods that perform singular value thresholding
\cite{CaiCanShe10} at every iteration.  While effective for small
problem instances, the computational expense of the SVD prevents the
method from being useful for large scale
problems. 

Recently, \citet{JaiMekDhi10} proposed a projected gradient descent algorithm
\texttt{SVP} (Singular Value Projection) that solves
\[
    \begin{aligned}
         & \min_{X \in \R^{n \times p} } && \norm{\A(X) - b}^2\\
         & \text{subject to} && \rank(X) \leq \r,
    \end{aligned}
\]
where $\norm{\cdot}$ is the $\ell_2$ vector norm and $\r$ is the input rank.
In the $(t+1)$th iteration, \texttt{SVP} updates $X^{t+1}$ as
the best rank $\r$ approximation to the gradient update $X^t - \mu
\A^\T(\A(X^t) - b)$, which is constructed from the SVD.
If $\rank(X^\star) = \r$, then \texttt{SVP} can recover $X^\star$ under
a similar RIP condition as the nuclear norm heuristic, and enjoys a linear
numerical rate of convergence. Yet \texttt{SVP} suffers from the expensive
per-iteration SVD for large problem instances.

Subsequent work of \citet{JaiNetSan13} proposes an alternating least
squares algorithm \texttt{AltMinSense} that avoids the per-iteration SVD.
\texttt{AltMinSense} factorizes $X$ into two factors $U \in \R^{n \times \r}, V
\in \R^{p \times \r}$ such that $X = UV^\T$ and minimizes
the squared residual $\norm{\A(UV^\T) - b}^2$ by updating $U$ and $V$
alternately. Each update is a least squares problem. The authors show that the
iterates obtained by \texttt{AltMinSense} converge to $X^\star$ linearly under
a RIP condition. However, the least squares problems are
often ill-conditioned, it is difficult to observe 
\texttt{AltMinSense} converging to $X^\star$ in practice.

As described above,
considerable progress has been made on algorithms for
rank minimization and certain semidefinite programming problems.  Yet truly efficient, scalable and
provably convergent algorithms have not yet been obtained.  In the specific setting that
$\Xstar$ is positive semidefinite, our algorithm exploits this
structure to achieve these goals. We note that 
recent and independent work of \citet{TuBocSol15} proposes a hybrid algorithm
called \emph{Procrustes Flow} (\texttt{PF}), which uses a few
iterations of \texttt{SVP} as initialization, and then applies gradient descent. 

\section{A Gradient Descent Algorithm for Rank Minimization}
\label{sec:algo}
Our method is described in Algorithm \ref{alg:gd}. It is parallel to the
\emph{Wirtinger Flow} (\texttt{WF}) algorithm for phase retrieval
\cite{CanLiSol14}, to recover a complex vector
$x \in \C^n$ given the squared magnitudes of its linear measurements $b_i =
\abs{\ip{a_i, x} }^2, \; i \in [m]$, where $a_1, \ldots, a_m \in \C^n$.
\citet{CanLiSol14} propose a first-order method to minimize the sum of
squared residuals
\begin{equation}
    \label{eq:rank1_phaseretrieval}
    f_{\texttt{WF}}(z) = \sum_{i=1}^n \left( \abs{\ip{a_i, z} }^2 -
        b_i \right)^2.
\end{equation}
The authors establish the convergence of \texttt{WF} to the global
optimum---given sufficient measurements, the iterates of
\texttt{WF} converge linearly to $x$ up to a global phase, with high probability.

If $z$ and the $a_i$s are real-valued, the function $f_{\texttt{WF}}(z)$ 
can be expressed as 
$$
f_{\texttt{WF}}(z) = \sum_{i=1}^n \left( z^\T a_i a_i^\T z - x^\T a_i a_i^\T x
\right)^2,$$ which is
a special case of $f(Z)$ where $A_i = a_i a_i^\T$ and each of $Z$ and $\Xstar$
are rank one. See Figure \ref{fig:rank1_objfunc} for an
illustration;  Figure \ref{fig:rank2_conv} shows the convergence rate of
our method.  Our
methods and results are thus generalizations of Wirtinger flow for
phase retrieval.

Before turning to the presentation of our technical results in
the following section, we present some intuition and remarks about how and why
this algorithm works.  For simplicity, let us assume that the rank is specified correctly.

Initialization is of course crucial in nonconvex optimization, as 
many local minima may be present.  To obtain a sufficiently 
accurate initialization, we use a spectral method, similar to those
used in \cite{NetJaiSan13,CanLiSol14}.  The starting point is
the observation that a linear combination of the constraint values
and matrices yields an unbiased estimate of the solution.

\begin{lemma}
    \label{lem:rank1_M_expectation}
    Let $M = \frac{1}{m} \sum_{i=1}^m b_i A_i$. Then $\frac{1}{2}\E(M)
    = \Xstar$, where the expectation is with respect to the randomness
    in the measurement matrices $A_i$.
\end{lemma}

Based on this fact, let $\Xstar = U^\star \Sigma {U^\star}^\T$ be the eigenvalue decomposition of $\Xstar$, where
$U^\star = [u^\star_1, \ldots, u^\star_\r]$ and $\Sigma = \diag(\sigma_1, \ldots, \sigma_\r)$
such that $\sigma_1 \geq \ldots \geq \sigma_\r$ are the nonzero eigenvalues of
$\Xstar$. Let $\Zstar = U^\star \Sigma^{\frac{1}{2}}$.
Clearly, $u^\star_s = \zstar_s / \norm{\zstar_s} $ is the top $s$th eigenvector of $\E(M)$
associated with eigenvalue $2\norm{\zstar_s}^2$.
Therefore, we initialize according to $z^0_s = \sqrt{\frac{|\lambda_s|}{2}} v_s$ where $(v_s,
\lambda_s)$ is the top $s$th eigenpair of $M$.
For sufficiently large $m$, it is reasonable to expect that 
$Z^0$ is close to $\Zstar$; this is confirmed by concentration of
measure arguments.

Certain key properties of $f(Z)$ will be seen to yield a linear rate
of convergence. In the analysis of convex functions, \citet{Nes04}
shows that for unconstrained optimization, the gradient descent scheme with sufficiently small step
size will converge linearly to the optimum if the objective function
is strongly convex and has a Lipschitz continuous gradient. However, these two
properties are global and do not hold for our objective function $f(Z)$.
Nevertheless, we expect that similar conditions
hold for the local area near $\Zstar$.  If so, then if we start close enough
to $\Zstar$, we can achieve the global optimum.

In our subsequent analysis, we establish the
convergence of Algorithm \ref{alg:gd} with a constant step size of the
form $\mu / \norm{\Zstar}^2_F$, where $\mu$ is a small constant. Since
$\norm{\Zstar}_F$ is unknown, we replace it by $\norm{Z^0}_F$.

\begin{figure}[tb]
    \centering
    \subfloat[]{
        \includegraphics[width=0.45\textwidth]{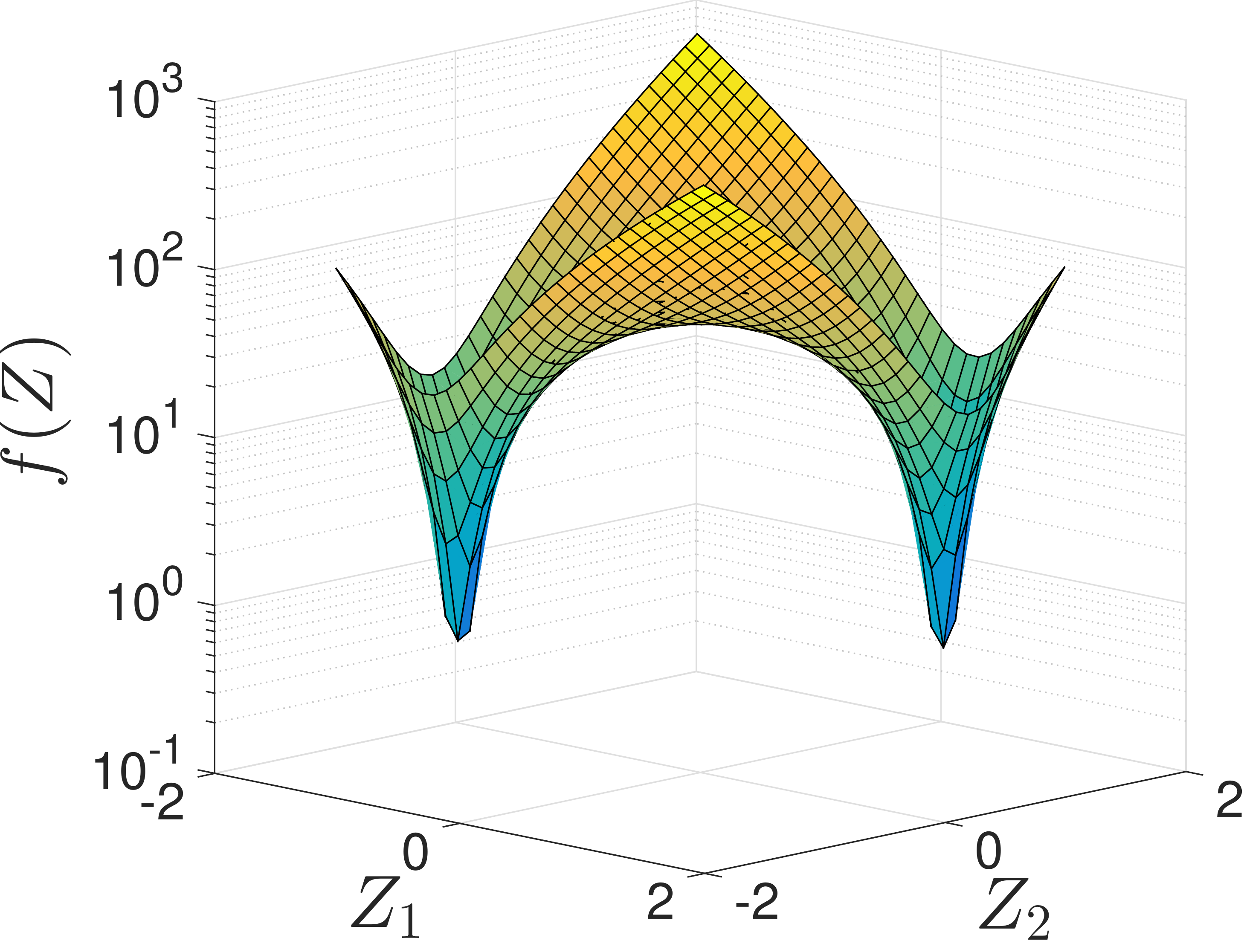}
        \label{fig:rank1_objfunc}
    }
    \hspace{1cm}
    \subfloat[]{
        \includegraphics[width=0.4\textwidth]{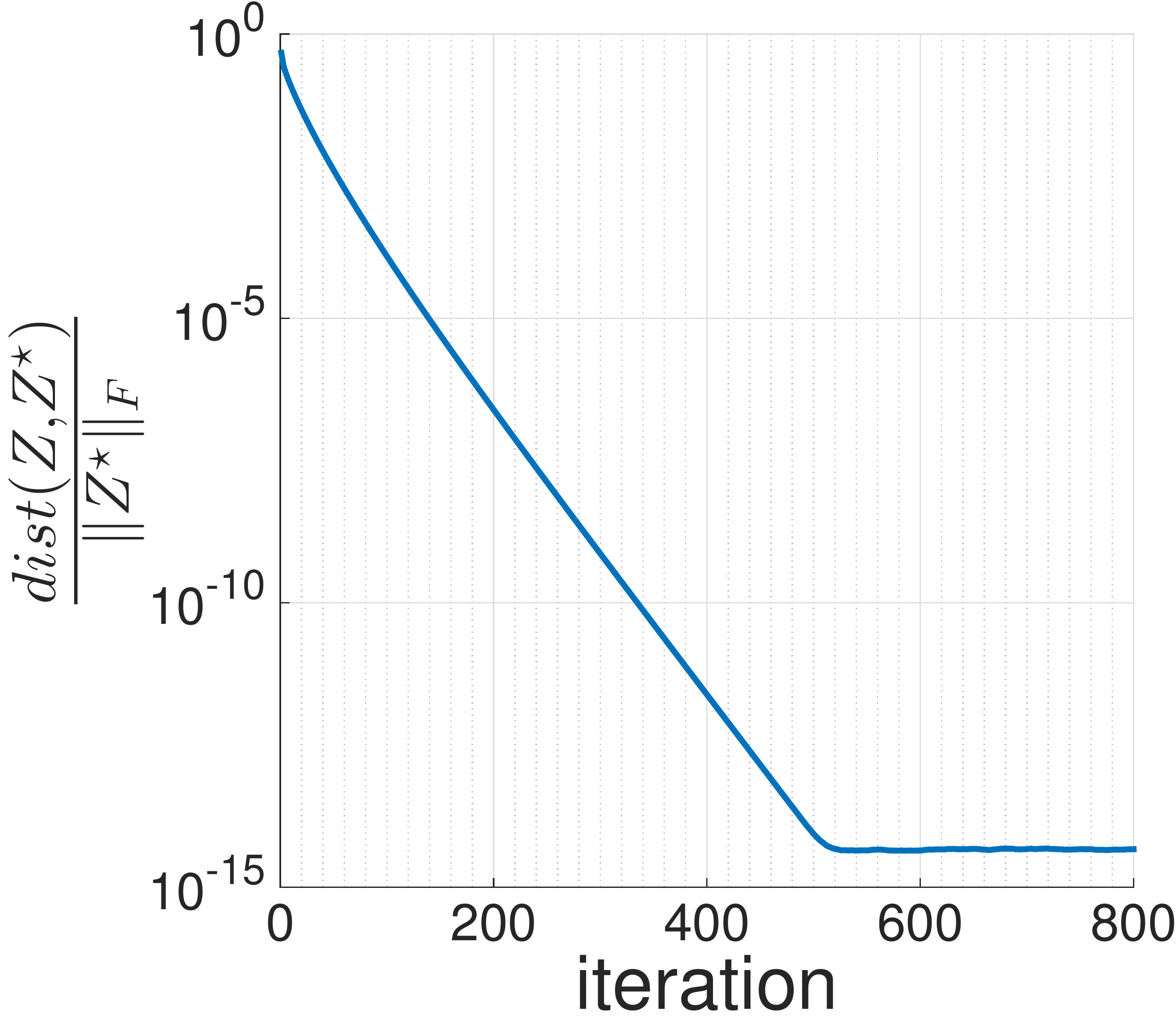}
        \label{fig:rank2_conv}
    }
    \caption{(a) An instance of $f(Z)$
        where $\Xstar \in \R^{2 \times 2}$ is rank-$1$ and $Z \in \R^2$. The underlying truth is
        $\Zstar = [1, 1]^\T$. Both $\Zstar$ and $-\Zstar$ are minimizers.
        (b) Linear convergence of the gradient scheme, for $n = 200$, $m=1000$
        and $\r=2$. The distance metric is given in Definition~\ref{def:dist}.
    }
\end{figure}

\begin{algorithm}[tb]
    \caption{Gradient descent for rank minimization}
    \label{alg:gd}
    \SetKwInOut{Input}{input}
    \Input{ $\{A_i, b_i\}_{i=1}^m, \r, \mu$ }
    \textbf{initialization} \\
    \hspace{0.5cm} Set $(v_1, \lambda_1), \ldots, (v_\r, \lambda_\r)$
    to the top $\r$ eigenpairs of $\frac{1}{m} \sum_{i=1}^{m} b_i A_i$ s.t. $\abs{\lambda_1} \geq \cdots \geq \abs{\lambda_\r}$\\
    \hspace{0.5cm} $Z^0 = [z^0_1, \ldots, z^0_\r]$ where $z^0_s =
    \sqrt{\frac{|\lambda_s|}{2}} \cdot v_s$, $s \in [\r]$\\
    \hspace{0.5cm} $k \leftarrow  0$\\
    \Repeat {convergence} {
        $\nabla f(Z^k) = \frac{1}{m} \sum\limits_{i=1}^m \left( \trace({Z^k}^\T
            A_i Z^k) - b_i \right)A_i Z^k $\\
        \vspace{3pt}
        $Z^{k+1} = Z^k - \dfrac{\mu}{\sum_{s=1}^\r |\lambda_s| / 2} \nabla f(Z^k)$\\
        $k \leftarrow k+1$
    }
    \SetKwInOut{Output}{output}
    \Output{$\hat{X} = Z^k{Z^k}^\T$ }
\end{algorithm}

\section{Convergence Analysis}
\label{sec:rst}
In this section we present our main result analyzing the gradient descent algorithm,
and give a sketch of the proof.  To begin, note that the 
symmetric decomposition of $X^\star$ is not unique, since $X^\star = (\Zstar U)(\Zstar U)^\T$ for any $\r \times \r$
orthonormal matrix $U$. Thus, the solution set is
\[
    \tS = \set{\Ztilde \in \R^{n\times \r} \; | \; \Ztilde = \Zstar
        U \;\; \text{for some $U$ with $UU^\T = U^\T U = I$}}.
\]
Note that $\| \Ztilde \|^2_F = \norm{X^\star}_*$ for any $\Ztilde \in
\tS$. We define the distance to the optimal solution in terms of this set.
\begin{defn} Define the distance between $Z$ and $\Zstar$ as
    \[d(Z, \Zstar) = \min_{UU^\T = U^\T U =I} \norm{Z - \Zstar U}_F =
        \min_{\Ztilde \in \tS} \big \| Z - \Ztilde \big \|_F. \] 
\label{def:dist}
\end{defn} 
Our main result for exact recovery is stated below, assuming that the
rank is correctly specified. Since  
the true rank is typically unknown in practice, one can start from a very low rank and
gradually increase it.  
\begin{theorem}
    \label{thm:rank1_exact} 
Let the condition number $\kappa = \sigma_1/\sigma_r$ denote the ratio of the largest to
the smallest nonzero eigenvalues of $\Xstar$. 
There exists a universal constant $c_0$ such that if $m \geq c_0
\kappa^2 \r^3 n \log n$, with high probability the initialization $Z^0$ satisfies
    \begin{equation}
        \label{eq:rank1_rst_thm_init}
            d(Z^0, \Zstar) \leq \sqrt{\frac{3}{16} \sigma_\r}.
    \end{equation}
Moreover, there exists a universal constant
$c_1$ such that when using constant step size $\mu / \norm{\Zstar}^2_F $ 
with $\mu \leq \dfrac{c_1}{\kappa n}$ and initial value
$Z^0$ obeying \eqref{eq:rank1_rst_thm_init}, the $k$th step of Algorithm~1 satisfies
\[ d(Z^k, \Zstar) \leq \sqrt{\frac{3}{16} \sigma_\r }\left(1 - \frac{\mu}{12\kappa\r}
        \right)^{k/2} \]
with high probability.
\end{theorem}
We now outline the proof, giving full details in the supplementary material.
The proof has four main steps.
The first step is to give a regularity condition 
under which the algorithm converges linearly if
we start close enough to $\Zstar$.  This provides
a local regularity property that is similar to the \citet{Nes04}
criteria that the objective function is
strongly convex and has a Lipschitz continuous gradient.

\begin{defn}
    \label{def:RC}
Let $\Zbar = \argmin_{\Ztilde \in \tS} \big \| Z - \Ztilde \big \|_F $ denote the matrix
closest to $Z$ in the solution set.
    \label{def:rank1_rc}
    We say that $f$ satisfies the \textit{regularity condition} $RC(\epsilon, \alpha, \beta)$ if
        there exist constants $\alpha$, $\beta$ such that for any $Z$
        satisfying $d(Z, \Zstar) \leq \epsilon$, we have
\[  \ip{\nabla f(Z), Z - \Zbar } \geq \frac{1}{\alpha} \sigma_\r \norm{Z - \Zbar
    }^2_F + \frac{1}{\beta \norm{\Zstar}^2_F } \norm{\nabla f(Z)}^2_F. \]
\end{defn}

Using this regularity condition, we show that the iterative step of
the algorithm moves closer to the optimum, if the current
iterate is sufficiently close.

\begin{theorem}
    \label{thm:rank1_geo}
Consider the update $Z^{k+1} = Z^k - \dfrac{\mu}{\norm{\Zstar}^2_F} \nabla f(Z^k)$.
If $f$ satisfies $RC(\epsilon, \alpha, \beta)$, $d(Z^{k},\Zstar)
\leq \epsilon$, and $0 < \mu < \min(\alpha /2, 2/\beta)$, then 
\[
    d(Z^{k+1}, \Zstar) \leq \sqrt{ 1 - \frac{2\mu }{\alpha \kappa \r} } d(Z^k, \Zstar).
\]
\end{theorem}

In the next step of the proof, we condition on two events that will be shown to
hold with high probability using concentration results. Let $\delta$
denote a small value to be specified later.
\begin{enumerate}
    \item[\textbf{A1}]\quad For any $u \in \R^{n}$ such that $\norm{u} \leq \sqrt{\sigma_1}$,
        \[\norm{\frac{1}{m}\sum\limits_{i=1}^m (u^\T A_i u)
        A_i - 2uu^\T} \leq \frac{\delta}{\r}.\]
    \item[\textbf{A2}]\quad For any $\Ztilde \in \tS$,
        \[ \norm{\dfrac{\partial^2 f (\Ztilde)  }{\partial \ztilde_s \partial \ztilde_k^\T }  - 
            \E\left[ \dfrac{\partial^2 f  (\Ztilde) }{\partial \ztilde_s
                \partial \ztilde_k^\T } \right] } \leq
    \frac{\delta}{\r}, \quad \mbox{for all $s, k$}
    \in [\r].\]
\end{enumerate}
Here the expectations are with respect to the random measurement
matrices.  Under these assumptions, we can show 
that the objective satisfies the regularity condition with high probability.
\vskip5pt

\begin{theorem}
    \label{thm:rank1_rc}
    Suppose that \textbf{A1} and \textbf{A2} hold.
    If $\delta \leq \frac{1}{16}\sigma_\r$, then $f$ satisfies the regularity
    condition $RC(\sqrt{\tfrac{3}{16}\sigma_\r}, 24, 513\kappa n )$
    with probability at least $1 - m C e^{-\rho n}$,
        where $C$, $\rho$ are universal constants. 
\end{theorem} 
Next we show that under \textbf{A1}, a good initialization can be found.

\begin{theorem} \label{thm:rank1_init_2}
    Suppose that \textbf{A1} holds. Let $\set{v_s, \lambda_s}_{s=1}^\r$ be the top
    $r$ eigenpairs of $M = \frac{1}{m}\sum\limits_{i=1}^m b_i A_i$ such
    that $\abs{\lambda_1} \geq \cdots
    \geq \abs{\lambda_\r}$. Let $Z^0 = [z_1, \ldots, z_\r]$ where $z_s = 
    \sqrt{\tfrac{\abs{\lambda_s}}{2}} \cdot v_s$, $s \in [\r]$. If $\delta \leq \tfrac{\sigma_\r}{4\sqrt{r}}$, then
    \[ d(Z^0, \Zstar) \leq \sqrt{3\sigma_\r/16}. \]
\end{theorem}

Finally, we show that conditioning on \textbf{A1} and \textbf{A2} is
valid since these events have high probability
as long as $m$ is sufficiently large.

\begin{theorem}
    \label{thm:rank1_sample_init}
    If the number of samples $m \geq \dfrac{42}{\min( \delta^2 / \r^2 \sigma^2_1,
        \; \delta / \r \sigma_1)}  n\log n$, then for
    any $u \in \R^n$ satisfying $\norm{u} \leq \sqrt{\sigma_1}$,
\[ \norm{ \frac{1}{m} \sum_{i=1}^m (u^\T A_i u) A_i - 2 u u^\T} \leq
             \frac{\delta}{\r} \] 
holds with probability at least $1 - m C e^{-\rho n} - \frac{2}{n^2}$, 
where $C$ and $\rho$ are universal constants. 
\end{theorem}

\begin{theorem}
\label{thm:rank1_sample_hess}
For any $x \in \R^n$, if $m \geq \dfrac{128}{\min(\delta^2 / 4\r^2\sigma^2_1,
     \; \delta / 2\r\sigma_1)} n\log n$, then for any $\Ztilde \in \tS$
    \[ \norm{\frac{\partial^2 f  (\Ztilde)}{\partial \ztilde_s \partial \ztilde_k^\T }  - 
            \E\left[ \frac{\partial^2 f (\Ztilde)}{\partial \ztilde_s \partial \ztilde_k^\T } \right] } \leq
        \frac{\delta}{\r}, \quad \text{for all} \; s, k \in [\r],
    \]
with probability at least $1 - 6me^{-n} - \frac{4}{n^2}$.
\end{theorem}

Note that since we need $\delta \leq \min\left( \frac{1}{16},
\frac{1}{4\sqrt{\r}} \right) \sigma_\r $, we have $\frac{\delta}{\r  \sigma_1} \leq 1$,
and the number of measurements required by our algorithm scales as $O(\r^3
\kappa^2  n \log n)$, while only $O(\r^2 \kappa^2 n \log n)$ samples are required by the
regularity condition. We conjecture this bound could be
further improved to be $O(\r n \log n)$;
this is supported by the experimental results presented below.

Recently, \citet{TuBocSol15} establish a tighter $O(r^2 \kappa^2 n)$ bound
overall. Specifically,
when only one single \texttt{SVP} step is used in preprocessing, the
initialization of \texttt{PF} is
also the spectral decomposition of $\frac{1}{2}M$. The authors show that $O(r^2
\kappa^2 n)$ measurements are sufficient for the initial solution to satisfy
$d(Z^0, \Zstar) \leq O(\sqrt{\sigma_r})$ with high probability, and
demonstrate an $O(rn)$ sample complexity for the regularity condition.

\section{Experiments}
\label{sec:expr}
In this section we report the results of experiments on synthetic
datasets.  We compare our gradient descent algorithm with nuclear norm
relaxation, \texttt{SVP} and \texttt{AltMinSense} for which we drop
the positive semidefiniteness constraint, as justified by the
observation in Section \ref{sec:conn_sdp}.  We use ADMM for the
nuclear norm minimization, based on the algorithm for the mixture
approach in \citet{TomHayKas10}; see Appendix \ref{sec:admm}.  For
simplicity, we assume that \texttt{AltMinSense}, \texttt{SVP} and the
gradient scheme know the true rank. Krylov subspace techniques such as
the Lanczos method could be used compute the partial eigendecomposition;
we use the randomized algorithm
of \citet{HalMarTro11} to compute the low rank SVD.
All methods are implemented in
MATLAB and the experiments were run on a MacBook Pro with a 2.5GHz
Intel Core i7 processor and 16 GB memory.

\subsection{Computational Complexity}
It is instructive to compare the per-iteration cost of the different
approaches; see Table \ref{table:periter_complexity}. Suppose that the density
(fraction of nonzero entries) of each $A_i$ is $\rho$. 
For \texttt{AltMinSense}, the cost of solving the least
squares problem is $O(mn^2\r^2 + n^3\r^3 + mn^2\r \rho)$.
The other three methods have $O(mn^2 \rho)$ cost to compute
the affine transformation.  For the nuclear norm approach, the $O(n^3)$ cost is
from the SVD and the $O(m^2)$ cost is due to the update of the dual variables.
The gradient scheme requires $2 n^2 \r$ operations to compute $Z^k {Z^k}^\T$ and 
to multiply $Z^k$ by $n \times n$ matrix to obtain the gradient.
\texttt{SVP} needs $O(n^2 \r)$ operations to compute the top $\r$ singular
vectors. However, in practice this partial SVD is more expensive than
the $2n^2 \r$ cost required for the matrix multiplies in the gradient scheme.

\begin{table}[htb]
    \centering
    \begin{tabular}{c l c}
        \toprule
        Method & Complexity\\ 
        \midrule
        nuclear norm minimization via ADMM &  $O(mn^2 \rho + m^2 + n^3)$\\
        gradient descent & $O(mn^2\rho) + 2n^2\r$ \\
        \texttt{SVP} & $O(mn^2\rho + n^2\r)$\\
        \texttt{AltMinSense} & $O(mn^2\r^2 + n^3\r^3 + mn^2\r \rho)$\\
        \bottomrule
    \end{tabular}
    \caption{Per-iteration computational complexities of different methods.}
    \label{table:periter_complexity}
\end{table}

Clearly, \texttt{AltMinSense} is the least efficient. For the other approaches, in the
dense case ($\rho$ large), the affine transformation dominates the computation.
Our method removes the overhead caused by the SVD. In the sparse case ($\rho$
small), the other parts dominate and our method enjoys a low cost.

\subsection{Runtime Comparison}
We conduct experiments for both dense and sparse measurement matrices.
\texttt{AltMinSense} is indeed slow, so we do not include it here.

In the first scenario, we randomly generate a $400 \times 400$ rank-$2$ matrix
$X^\star = x x^\T + y y^\T$ where $x, y \sim \N(0, I)$.  We also
generate $m=6n$ matrices $A_1,
\ldots, A_m$ from the GOE, and then take $b = \A(X^\star)$. We report the relative
error measured in the Frobenius norm defined as $ \|\wh{X} -  X^\star \|_F/ \|
X^\star \|_F$.  For the nuclear norm approach, we set the regularization
parameter to $\lambda = 10^{-5}$. We test three values $\eta = 10, 100, 200$ for the penalty
parameter and select $\eta = 100$ as it leads to the fastest convergence.  
Similarly, for \texttt{SVP} we evaluate the three values $5\times 10^{-5}, 10^{-4}, 2\times 10^{-4}$ for the
step size, and select $10^{-4}$ as the largest for which
\texttt{SVP} converges. For our approach, we test the three values
$0.6, 0.8, 1.0$ for $\mu$ and select $0.8$ in the same way.

In the second scenario, we use a more general and practical setting.  We
randomly generate a rank-$2$ matrix $X^\star \in \R^{600 \times 600}$ as before. We generate
$m=7n$ sparse $A_i$s whose entries are i.i.d.~Bernoulli:
\[  
    (A_i)_{jk} = 
    \begin{cases}
        1 & \text{with probability} \; \rho, \\
        0 & \text{with probability} \; 1-\rho,
    \end{cases}
\]
where we use $\rho=0.001$. For all the methods we use the same strategies as before
to select parameters.
For the nuclear norm approach, we try three values $\eta = 10, 100, 200$ and
select $\eta = 100$.
For $\texttt{SVP}$, we test the three values $5 \times 10^{-3},  2 \times
10^{-3},  10^{-3}$ for the step size and select $10^{-3}$.
For the gradient algorithm, we check the three values $0.8, 1, 1.5$ for $\mu$ and
choose $1$.

The results are shown in Figures \ref{fig:nips_rank2_runtime} and
\ref{fig:nips_rank2_runtime_sparse}. In the dense case, our method is faster
than the nuclear norm approach and slightly outperforms \texttt{SVP}. In the sparse
case, it is significantly faster than the other approaches.

\begin{figure}[tb]
    \hspace{-0.9cm}
    \subfloat[] {
        \includegraphics[width=0.34\textwidth]{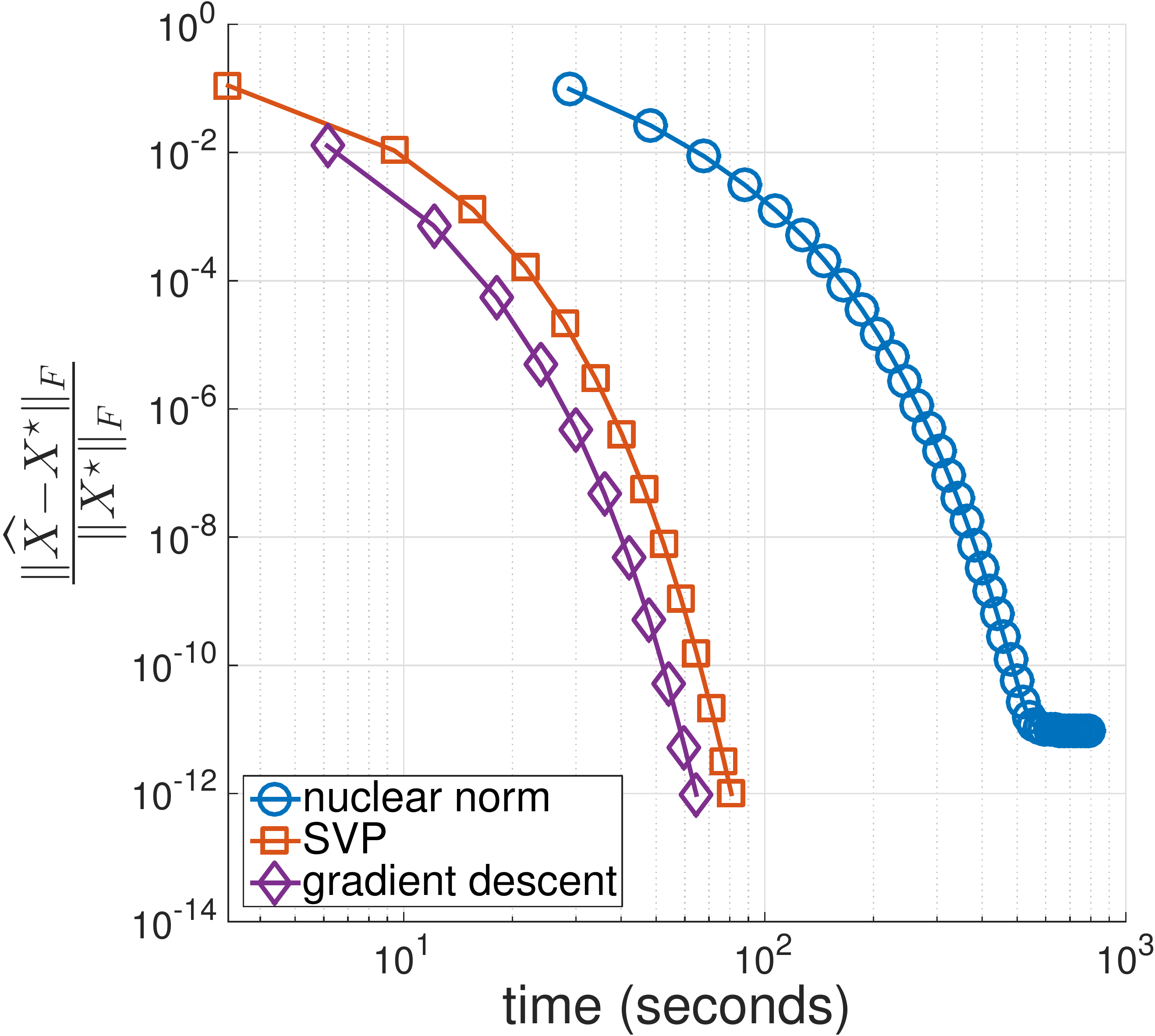}
        \label{fig:nips_rank2_runtime}
    }
    \subfloat[]{
        \includegraphics[width=0.33\textwidth]{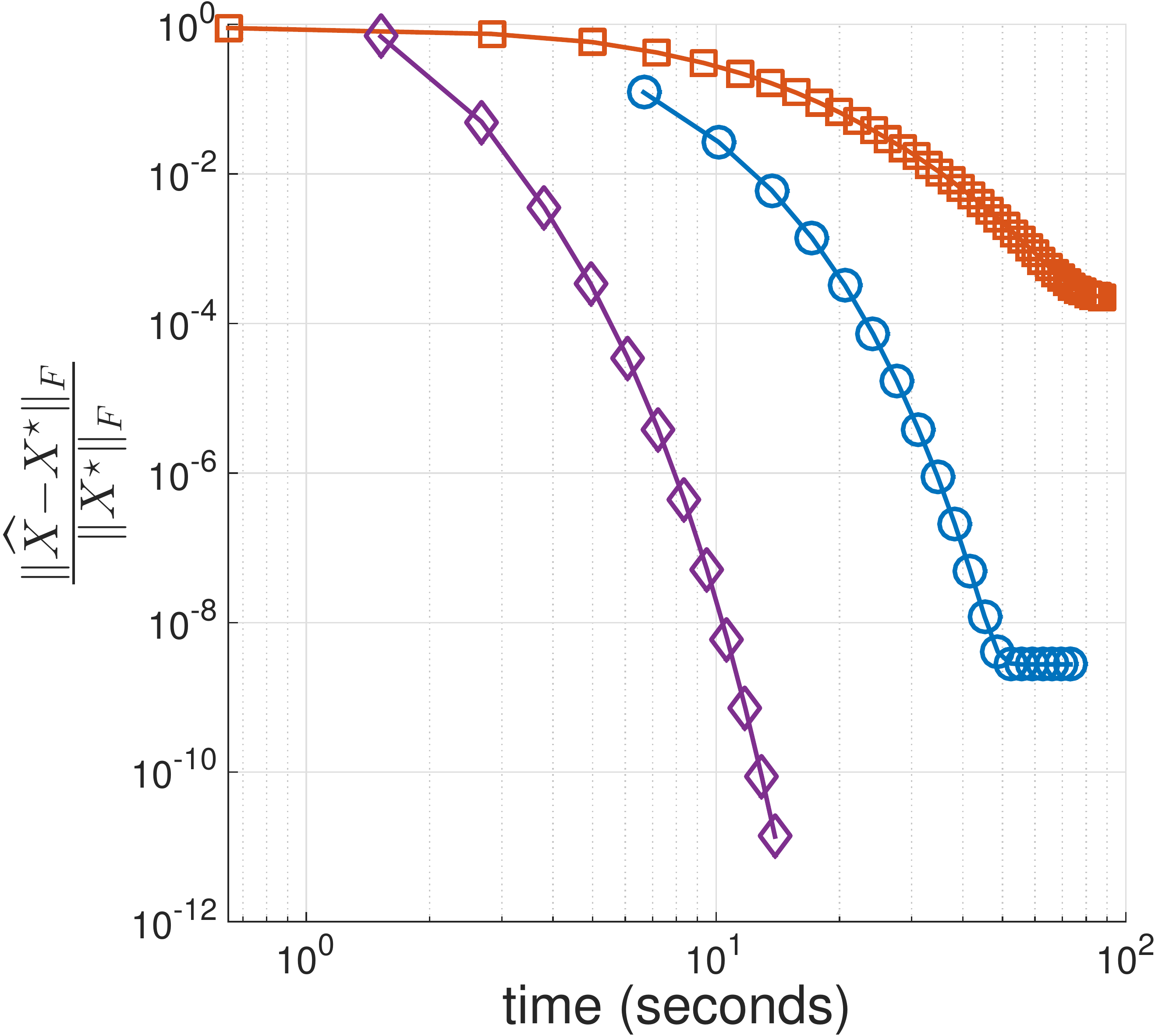}
        \label{fig:nips_rank2_runtime_sparse}
    }
    \subfloat[]{
        \includegraphics[height=0.3\textwidth]{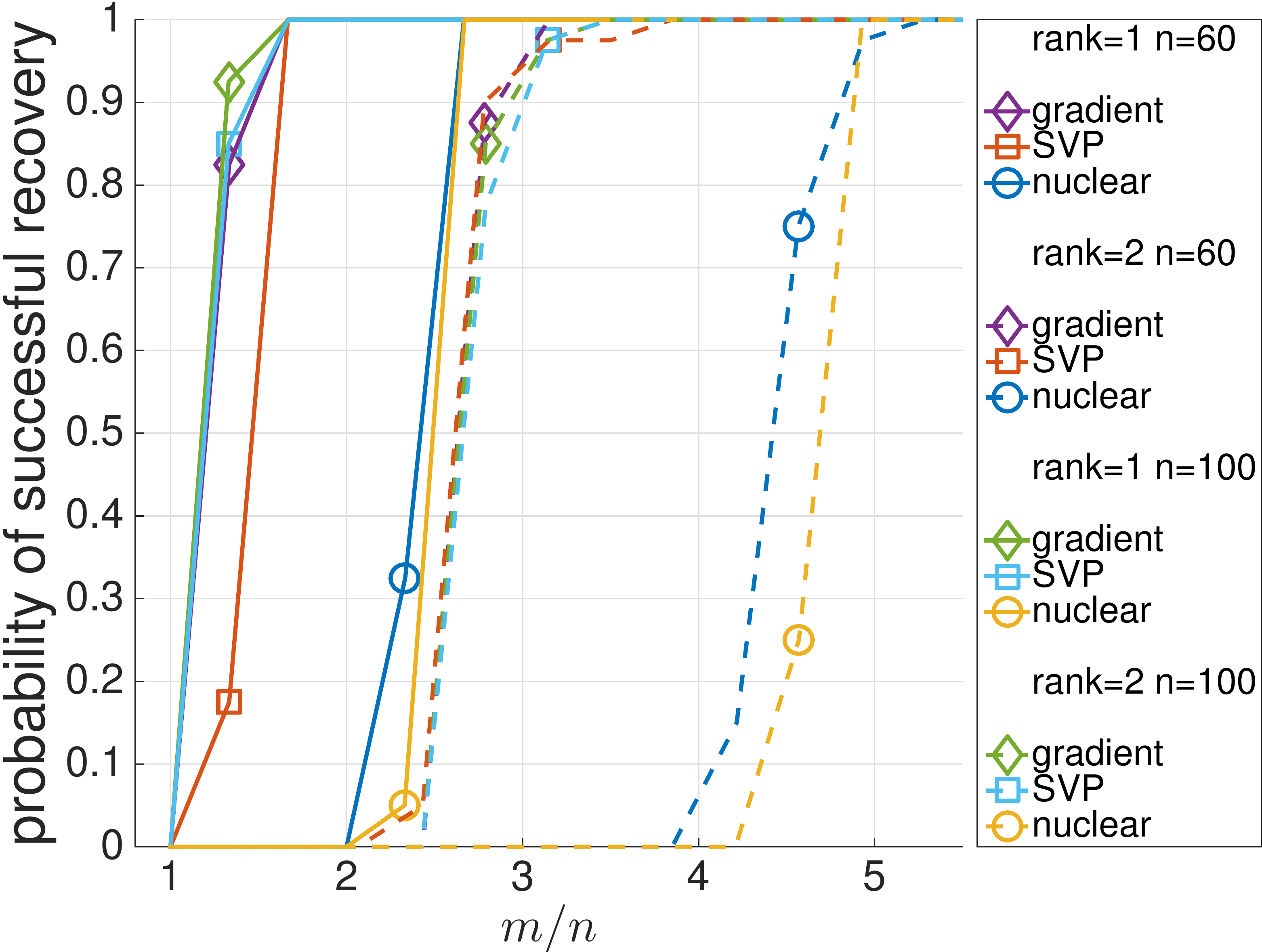}
        \label{fig:rank2_sample_phasetransition}
    }
    \caption{(a) Runtime comparison where $\Xstar \in \R^{400 \times 400}$ is
        rank-$2$ and $A_i$s are dense. (b)
        Runtime comparison where $\Xstar \in \R^{600 \times 600}$ is rank-$2$
        and $A_i$s are sparse. (c) Sample complexity comparison. }
\end{figure}

\subsection{Sample Complexity}
We also evaluate the number of measurements required by each method to
exactly recover $X^\star$, which we refer to as the \emph{sample complexity}.
We randomly generate the true matrix $\Xstar \in \R^{n \times n}$ and compute the
solutions of each method given $m$ measurements, where the $A_i$s are
randomly drawn from the GOE.  A solution with relative error below $10^{-5}$ is
considered to be successful. We run 40 trials and compute the
empirical probability of successful recovery. 

We consider cases where $n=60$ or $100$ and $\Xstar$ is of rank one or two.
The results are shown in Figure \ref{fig:rank2_sample_phasetransition}. 
For \texttt{SVP} and our approach, the phase transitions happen around $m=1.5n$
when $\Xstar$ is rank-$1$ and $m=2.5n$ when $\Xstar$ is rank-$2$.
This scaling is close to the number of degrees of freedom in each case;
this confirms that the sample complexity scales linearly with the rank $\r$.
The phase transition for the nuclear norm approach occurs later. 
The results suggest that the sample complexity of our method should also scale as $O(r n
\log n)$ as for \texttt{SVP} and the nuclear norm approach \cite{JaiMekDhi10,
RecFazPar10}.

\section{Conclusion}
\label{sec:conclude}

We connect a special case of affine rank minimization to a class of
semidefinite programs with random constraints. Building on a recently
proposed first-order algorithm for phase retrieval \cite{CanLiSol14},
we develop a gradient descent procedure for rank minimization and
establish convergence to the optimal solution with $O(\r^3 n \log n)$
measurements. We conjecture that $O(\r n \log n)$ measurements are
sufficient for the method to converge, and that the conditions on the
sampling matrices $A_i$ can be significantly weakened. More broadly,
the technique used in this paper---factoring the semidefinite matrix
variable, recasting the convex optimization as a nonconvex
optimization, and applying first-order algorithms---first proposed by
Burer and Monteiro \cite{BurMon03}, may be effective for a much wider
class of SDPs, and deserves further study.

\section*{Acknowledgements}

Research supported in part by NSF grant IIS-1116730 and ONR grant
N00014-12-1-0762. The authors thank Afonso Bandeira,
Ryota Tomioka and the authors of \citet{TuBocSol15} for helpful comments on this work.

\bibliography{main}
\bibliographystyle{plainnat}
\newpage
\appendix
\section{Proof of Lemma \ref{lem:rank1_M_expectation}}
    Let $A = (a_{ij})$ be a random matrix that is GOE distributed;
    thus $a_{ij} \sim \N(0,1)$ for $i\neq j$ and $a_{ii} \sim \N(0,2)$. We have
    $\E(M) = \sum_{s=1}^\r \E( ({\zstar_s}^\T A \zstar_s) A )$. Hence, it
    suffices to show that $\E( (x^\T A x) A ) = 2 xx^\T$ for any $x \in \R^n$.
    The $(i,j)$ entry of $ (x^\T A x) A $ has expected value
        \begin{align*}
            \E( (x^\T A x) a_{ij}) & =  \E\left( \sum_k \sum_l x_k x_l a_{kl} a_{ij} \right ) \\
            & =  \sum_k \sum_l x_k x_l \E( a_{kl} a_{ij} ) \\
            & =  \sum_k \sum_l x_k x_l \cdot \begin{cases}
                0 & \text{if} \; (k, l) \neq (i,j) \wedge (k, l) \neq (j, i)\\
                \E(a_{kl}^2) & \text{otherwise}
            \end{cases} \\
            & =  \begin{cases} 
                2 x_i x_j \E(a^2_{ij}) & \text{if} \; i \neq j \\
                x^2_i \E(a^2_{ii}) & \text{otherwise}
            \end{cases} \\
            & =  \begin{cases} 
                2 x_i x_j  & \text{if} \; i \neq j,  \\
                2 x^2_i   & \text{otherwise},
            \end{cases} \\
        \end{align*}
    where we use that the variance of $a_{ii}$ is 2 and the variance of
    $a_{ij}$ is 1 for any $ i \neq j$. In matrix form, this is 
    $\E((x^\T A x) A) = 2 xx^\T$.

\section{Ingredients}
\label{sec:ingredients}
We first present some technical lemmas that will be needed later. Recall
Definition \ref{def:RC} that for any $Z$, $\Zbar = \argmin_{\Ztilde \in \tS}
\big \| Z - \Ztilde \big \|_F $.  Let $H = Z - \Zbar$. 
The $s$th column of $Z$, $\Zbar$, $\Zstar$, $H$ are denoted by $z_s$, $\zbar_s$, $\zstar_s$, $h_s$ respectively.
We shall use the following formulas for the gradient and second order partial
derivatives:
\begin{align*}
    \nabla f(Z) &= \frac{1}{m} \sum_{i=1}^m \left( \trace(H^\T A_i H) + 2\trace(\Zbar^\T A_i H)  \right) ( A_i H + A_i \Zbar )  ,\\
    \frac{ \partial^2  f(Z)  } {\partial z_s \partial z_s^\T } &=
    \frac{1}{m}\sum_{i=1}^m \left( 2 A_i z_s z_s^\T A_i^\T + \left(\trace(Z^\T
            A_i Z) - b_i\right) A_i\right), \quad \forall s \in [\r], \\
    \frac{ \partial^2  f(Z)  } {\partial z_s \partial z_k^\T } &=
\frac{1}{m}\sum_{i=1}^m 2 A_i z_s z_k^\T A_i^\T, \quad \forall s, k \in [\r] \; 
\text{such that } \; s \neq k.
\end{align*}

The next ingredient we need is the expectation of the second order partial
derivatives with respect to the random measurement matrices.
\begin{lemma} Let $A = (a_{ij})$ be a GOE distributed random matrix. For any two
    fixed vectors $x$ and $y$, we have
$\E \left[ A x y A \right] = x^\T y I + yx^\T$.
\label{lem:rank2_hess_cross_expectation}
\end{lemma}
\begin{proof}
    The expectation of $(i,j)$ entry of $Axy^\T A$ is
    \[\E[ (Axy^\T A)_{ij}] = \E\left( \sum_{k\;l} a_{ik} a_{jk} x_k y_l  \right ).\]
    If $i = j$, then we have
    \[ \E[ (Axy^\T A)_{ii} ]= \E\left( \sum_k  a^2_{ik} x_k y_k  \right ) =
        \sum_k x_k y_k + x_i y_i, \]
        since $\Var(a^2_{ii}) = 2$ and $\Var(a^2_{ik}) = 1$ if $k \neq i$.
    On the other hand, if $i \neq j$, then
    \[ \E[ (Axy^\T A)_{ij} ] = \E \left( \sum_{kl} a_{ik} a_{jl} x_k y_l \right)
        = \E( a^2_{ij} x_j y_i) = x_j y_i. \]
    Therefore, $\E(Axy^\T A) = x^\T y I + yx^\T$. 
\end{proof}

\begin{lemma} For all $s \in [r]$, it holds that $\E \left[ \dfrac{ \partial^2  f(Z)  } {\partial z_s \partial z_s^\T }
    \right] = 2\norm{z_s}^2 I + 2 z_s z_s^\T + 2ZZ^\T - 2X^\star$
    and $ \E \left[ \dfrac{ \partial^2 f(Z)  } {\partial z_s \partial z_k^\T }\right] =
2z^\T_s z_k I + 2 z_k z_s^\T $ for all $ k \in [r] $ such that $ k
\neq s$, where the expectation is over the random measurement matrices.
\label{lem:rank2_hess_expectation}
\end{lemma}
\begin{proof}
    The case where $k \neq s$ is a direct result of Lemma
    \ref{lem:rank2_hess_cross_expectation}. For the other case,
    let $A = (a_{ij})$ be a GOE distributed random matrix. 
    It follows from Lemma \ref{lem:rank1_M_expectation} that
        \[ \E \left[ \frac{ \partial^2 f(Z)} {\partial z_s \partial z_s^\T } \right]
    = 2 \E( Az_s z_s^\T A) + 2ZZ^\T - 2X^\star. \]
    By Lemma \ref{lem:rank2_hess_cross_expectation}, we have
    \[ \E( Az_sz_s^\T A) = \norm{z_s}^2 I + z_sz_s^\T. \]
    Substituting this back into the above equation, we obtain the lemma.
\end{proof}

We next recall a concentration result for the operator (spectral) norm of the
random measurement matrices.
\begin{lemma}(\citet[Theorem 1]{LedRid10})
    \label{lem:GOE_operator_norm}
  There exists two absolute constants $C$ and $\rho = \frac{1}{\sqrt{8}C}$
 such that with probability at least $1 - C e^{-\rho n}$, 
\[ \norm{A_i} \leq 3\sqrt{n}.\]
\end{lemma}
A tighter upper bound is actually given in the \textit{Tracy-Widow
law}: w.h.p. $ \norm{A_i} = O(2\sqrt{n} + n^{1/6})$.
\begin{coro}
    \label{lem:Ai_operator_norm}
    With probability at least $ 1 - m C e^{-\rho  n } $, the average of the
    squared operator norm of the random measurement matrices is upper bounded by $9n$.
\end{coro}
\begin{proof}
Applying a union bound we have
    \[
        \begin{aligned} 
            \P\left(\frac{1}{m} \sum_{i=1}^m \norm{A_i}^2 \leq 9n \right) 
            & \geq && 
            \P\left(\forall i,\; \norm{A_i} \leq 3\sqrt{n} \right)  \\
            & \geq && 1 - \sum_{i=1}^m \P\left( \norm{A_i} > 3 \sqrt{n} \right) \\
            & \geq && 1 - m C e^{-\rho n },
        \end{aligned} 
    \]
where we use Lemma \ref{lem:GOE_operator_norm} in the last line.
\end{proof}

The following two technical lemmas are important tools for us.
Define the set
 \[ E(\epsilon) = \set{ Z \;|\; d(Z,\Zstar) \leq \epsilon }.\]

\begin{lemma}
    \label{lem:rank2_lem_M}
    Suppose that \textbf{A1} holds: $\norm{ \frac{1}{m} \sum_{i=1}^m (u^\T A_i u) A_i - 2 u u^\T }
    \leq \frac{\delta}{\r}$, for any $u$ such that $\norm{u} \leq \sqrt{\sigma_1}$. 
    If $\delta \leq \frac{1}{16}\sigma_\r$, then for any $ Z \in E\left(\sqrt{
    \frac{3}{16} \sigma_\r}  \right)$ it holds that
   \[  
       2\norm{HH^\T}^2_F -\delta \norm{H}^2_F \leq \frac{1}{m} \sum_{i=1}^m
       \trace(H^\T A_i H)^2 \leq \delta \norm{H}^2_F + 2\norm{HH^\T}^2_F. 
   \]
\end{lemma}
\begin{proof}
   Let $h_s$ be the $s$th column of $H$.
   Since $\max_{s \in [\r]} \norm{h_s}_2  \leq \norm{H}_F \leq 
   \sqrt{\frac{3}{16}\sigma_\r} \leq \sqrt{\sigma_1}$, it follows from the assumption
   of the lemma that
   \[\norm{ \frac{1}{m} \sum_{i=1}^m (h_s^\T A_i h_s) A_i - 2 h_s h_s^\T } \leq \frac{\delta}{\r}, \quad s = 1, \ldots, \r.\]
   By the triangle inequality, we have
   \[\norm{ \frac{1}{m} \sum_{i=1}^m \sum_{s=1}^\r (h_s^\T A_i h_s) A_i - 2
           \sum_{s=1}^\r h_s h_s^\T } \leq \delta \]
   and consequently
   \[  -\delta \norm{h_s}^2 \leq h^\T_s \left( \frac{1}{m} \sum_{i=1}^m \trace(H^\T A_i H) A_i  - 2
           HH^\T \right) h^\T_s \leq \delta \norm{h_s}^2, \; s
       = 1, \ldots, \r, \]
   where we replace $\sum\limits_{s=1}^\r h^\T_s A_i h_s$ by $\trace(H^\T A_i
   H)$ and $\sum_{s=1}^\r h_s h_s^\T$ by $HH^\T$. Taking the sum of
   the above inequalities, we obtain
   \[  -\delta \norm{H}^2_F \leq \frac{1}{m} \sum_{i=1}^m \trace(H^\T A_i
           H)^2   - 2 \trace(H^\T H H^\T H) \leq \delta \norm{H}^2_F. \]
   Note that $\trace(H^\T H H^\T H) = \norm{HH^\T}^2_F$. Therefore,
   \[  2\norm{HH^\T}^2_F -\delta \norm{H}^2_F \leq \frac{1}{m} \sum_{i=1}^m
       \trace(H^\T A_i H)^2 \leq \delta \norm{H}^2_F + 2\norm{HH^\T}^2_F. \]
\end{proof}

\begin{lemma}
    \label{lem:rank2_lem_hess}
    Suppose that \textbf{A2} holds: for any $\Ztilde$ such that $\Ztilde\Ztilde^\T = X^\star$ we have
\begin{equation}
\label{eq:assn}    
        \norm{ \dfrac{\partial^2 f(\Ztilde)  }{\partial \ztilde_s \partial \ztilde_k^\T} - \E \left[
    \dfrac{\partial^2  f(\Ztilde)  }{\partial \ztilde_s \partial \ztilde_k^\T} \right]}
    \leq \frac{\delta}{\r},
  \; \; s,k = 1,\ldots, \r. 
\end{equation}
  Then
  \[  \left(\sigma_\r - \frac{\delta}{2} \right) \norm{H}^2_F + \norm{H^\T \Zbar}^2_F
        \leq \frac{1}{m} \sum_{i=1}^m \trace(H^\T A_i \Zbar)^2  \leq
        \left( \sigma_1 + \frac{\delta}{2} \right) \norm{H}^2_F + \norm{H^\T \Zbar}^2_F.
    \]
\end{lemma}
\begin{proof}
    Our goal is to bound $\frac{1}{m} \sum\limits_{i=1}^m \trace(H^\T A_i
    \Zbar)^2$. This can be expanded as
    \[\frac{1}{m}\sum_{i=1}^m \left( \sum_{s=1}^\r (h^\T_s A_i \zbar_s) \right)^2 = 
        \frac{1}{m}\sum_{i=1}^m \sum_{s=1}^\r (h^\T_s A_i x_s)^2 + 
        \frac{1}{m}\sum_{i=1}^m \sum_{s<k} 2 (h^\T_s A_i x_s)(h^\T_k A_i x_k). \]

    We first bound the sum of the quadratic terms. For any $s \in [\r]$, we have
    \[
        \begin{aligned} 
            \frac{\partial^2  f(\Zbar)  }{\partial \zbar_s \partial \zbar_s^\T} & = \frac{1}{m}
            \sum_{i=1}^m 2 A_i \zbar_s \zbar_s^\T A_i, \\
            \E\left[  \frac{\partial^2 f(\Zbar)   }{\partial \zbar_s \partial \zbar_s^\T}
            \right] & = 2\norm{\zbar_s }^2 I + 2 \zbar_s \zbar_s^\T.
        \end{aligned}
    \]
    It follows from assumption \eqref{eq:assn} that for any $s \in [\r]$,
    \[
        -\frac{\delta}{\r} \norm{h_s}^2 \leq \frac{1}{m} \sum_{i=1}^m 2 (h_s^\T A_i \zbar_s)^2 -   2\norm{\zbar_s }^2
        \norm{h_s}^2 - 2 (h^\T_s \zbar_s)^2 \leq \frac{\delta}{\r} \norm{h_s}^2.
    \]
    Taking the sum of above inequalities, we obtain
    \begin{equation}
        \label{eq:rank2_lem_hess_helper1}
        -\frac{\delta}{2\r} \sum_{s=1}^\r \norm{h_s}^2 \leq \frac{1}{m} \sum_{i=1}^m
        \sum_{s=1}^\r  (h_s^\T A_i \zbar_s)^2 -   \sum_{s=1}^\r \norm{\zbar_s }^2
        \norm{h_s}^2 -  \sum_{s=1}^\r (h^\T_s \zbar_s)^2 \leq \frac{\delta}{2\r} \sum_{s=1}^\r
        \norm{h_s}^2.
    \end{equation}
    Similarly, we bound the sum of the cross terms. For any fixed $s,
    k$ such that $s \neq k$, we have
    \[
        \begin{aligned} 
            \frac{\partial^2 f(\Zbar)  }{\partial \zbar_s \partial \zbar_k^\T} & = \frac{1}{m}
                f(\Zbar)
            \sum_{i=1}^m 2 A_i \zbar_s \zbar_k^\T A_i, \\
            \E\left[  \frac{\partial^2  f(\Zbar)  }{\partial \zbar_s \partial \zbar_k^\T}
            \right] & = 2 \zbar_s^\T \zbar_k I + 2 \zbar_k \zbar_s^\T,
        \end{aligned}
    \]
    and consequently
    \begin{align}
        \label{eq:rank2_lem_hess_helper2}
        - \frac{\delta}{\r} \sum_{s<k} \norm{h_s} \norm{h_k} & \leq \frac{1}{m}
        \sum_{i=1}^m \sum_{s<k} 2 (h_s^\T A_i
        \zbar_s)(h_k^\T A_i \zbar_k) -   2 \sum_{s<k} \zbar_s^\T \zbar_k h_s^\T h_k -
        2\sum_{s<k} h^\T_s \zbar_k \zbar_s^\T
        h_k \\
\nonumber
& \leq \frac{\delta}{\r} \sum_{s<k} \norm{h_s}\norm{h_k}.
    \end{align}
    We combine equations \eqref{eq:rank2_lem_hess_helper2} and
    \eqref{eq:rank2_lem_hess_helper1} to get
    \begin{equation}
        \label{eq:inner_step_lem_sank2_hess}
        - \frac{\delta}{2\r} \sum_{sk} \norm{h_s} \norm{h_k} \leq \frac{1}{m}
        \sum_{i=1}^m \trace(H^\T A_i \Zbar)^2 - 
         \sum_{sk} \zbar_s^\T \zbar_k h_s^\T h_k -
        \sum_{sk} h^\T_s \zbar_k \zbar_s^\T
    h_k \leq \frac{\delta}{2\r} \sum_{sk} \norm{h_s}\norm{h_k}. 
\end{equation}
    Note that
    $\sum_{sk} h^\T_s \zbar_k \zbar_s^\T h_k = \trace(H^\T \Zbar H^\T \Zbar)$,
        $\sum_{sk} \zbar_s^\T \zbar_k h_s^\T h_k = \norm{\Zbar H^\T}^2_F$ and
    \[
        \sum_{sk} \norm{h_s}\norm{h_k} = \left( \sum_{s=1}^\r \norm{h_s}
        \right)^2 \leq \r \sum_{s=1}^\r \norm{h_s}^2 = \r \norm{H}^2_F.
    \]
    By Lemma \ref{lem:rank2_lem_hess_helper3}, $\trace(H^\T \Zbar H^\T \Zbar) = \norm{H^\T \Zbar}^2_F$.
    Replacing those terms in equation \eqref{eq:inner_step_lem_sank2_hess} gives us
    \[  - \frac{\delta}{2} \norm{H}^2_F +
        \norm{\Zbar H^\T}^2_F + \norm{H^\T \Zbar}^2_F
        \leq \frac{1}{m} \sum_{i=1}^m \trace(H^\T A_i \Zbar)^2  \leq
        \frac{\delta}{2} \norm{H}^2_F + \norm{\Zbar H^\T}^2_F + \norm{H^\T \Zbar}^2_F.
    \]
    Finally, we obtain the claim by noticing that
    \[\sqrt{\sigma_\r} \norm{H}_F \leq \norm{ \Zbar H^\T}_F \leq
    \sqrt{\sigma_1} \norm{H}_F, \]
    where
    $\sqrt{\sigma_1} = \sigma_{\max} (\Zbar) \geq \cdots \geq
    \sigma_{\min} (\Zbar) = \sqrt{\sigma_\r}$ are the singular values of $\Zbar$.
\end{proof}

\begin{lemma} 
    \label{lem:rank2_lem_hess_helper3}
    $\trace(H^\T \Zbar H^\T \Zbar) = \norm{H^\T \Zbar}^2_F$.
\end{lemma}
\begin{proof}
    Let $\Ubar = \argmin_{UU^\T = U^\T U = I} \norm{Z - \Zstar U}^2_F =
    \argmax_{UU^\T = U^\T U = I} \ip{U, {\Zstar}^\T Z}$.  Note that $\ip{A, B}
    \leq \norm{A}_* \norm{B} $ for any matrices $A, B$ that are of the same
    size. The equality holds when $B = U_A V_A^\T$ where $A = U_A \Sigma_A
    V_A^\T$ is the SVD of $A$. Hence, $\Ubar = \tilde{U} \tilde{V}^\T$ where
    $\tilde{U} \tilde{S} \tilde{V}^\T$ is the SVD of ${\Zstar}^\T Z$; $\Zbar =
    \Zstar \Ubar$.
    Therefore, $Z^\T \Zbar = Z^\T \Zstar \Ubar = \tilde{V} \tilde{S}
    \tilde{V}^\T$ is symmetric and positive semidefinite. Thus, $H^\T \Zbar = Z^\T \Zbar - {\Zbar}^\T
    \Zbar$ is also symmetric. This implies that $\trace(H^\T \Zbar H^\T \Zbar) =
    \norm{H^\T \Zbar}^2_F$.
\end{proof}

\section{Linear Convergence}
\textbf{Proof of Theorem \ref{thm:rank1_geo}}

Let $H^k = Z^k - \Zbar^k$. Then we have that
 \begin{align*}
        & \norm{Z^{k+1} - \Zbar^k }^2_F  =  \norm{Z^k - \frac{\mu}{\norm{\Zstar}^2_F}\nabla
            f(Z^k) - \Zbar^k}^2_F \\
       & \qquad  = \norm{H^k}^2_F + \frac{\mu^2}{\norm{\Zstar}^4_F} \norm{ \nabla
            f(Z^k)}^2_F -  \frac{2\mu}{\norm{\Zstar}^2_F}
        \ip{ \nabla f(Z^k), H^k} \\
        & \qquad  \leq   \norm{H^k}^2_F + \frac{\mu^2}{\norm{\Zstar}^4_F} \norm{ \nabla
            f(Z^k)}^2_F -
        \frac{2\mu}{\norm{\Zstar}^2_F}\left(\frac{1}{\alpha} \sigma_\r \norm{H^k}^2_F + \frac{1}{\beta \norm{\Zstar}^2_F}\norm{\nabla
                f(Z^k)}^2_F \right)\\
        & \qquad  =  \left( 1 - \frac{2\mu}{\alpha}\cdot
                \frac{\sigma_\r}{\sum_{s=1}^\r\sigma_s} \right) \norm{H^k}^2_F + \frac{\mu (\mu -
            2/\beta)}{\norm{\Zstar}^4_F} \norm{\nabla f(Z^k)}^2_F\\
        & \qquad  \leq \left( 1 - \frac{2\mu}{\alpha} \cdot
            \frac{\sigma_\r}{\r \sigma_1}\right)\norm{H^k}^2_F \\
        & \qquad  =  \left(1 - \frac{2\mu}{\alpha \kappa \r}\right) d(Z^k, \Zstar)^2,
\end{align*}
where we use the definition of $RC(\epsilon, \alpha, \beta)$ in the third line,
$\norm{\Zstar}^2_F = \norm{X^\star}_* = \sum_{s=1}^\r \sigma_s$ in the third to last
line and $0 < \mu < \min\set{ \alpha / 2,  2 / \beta }$ in the second to last line.
Therefore, 
\[ d(Z^{k+1}, \Zstar) = \min_{\Ztilde \in \tS} \norm{Z^{k+1} - \Ztilde}^2_F \leq
    \sqrt{1 - \frac{2\mu}{\alpha \kappa \r} }d(Z^k, \Zstar).\]

\section{Regularity Condition}
As mentioned before, \citet[Theorem 2.1.11]{Nes04} shows
that the gradient scheme converges linearly under a condition
similar to the regularity condition, which is satisfied if the function is
strongly convex and has a Lipschitz continuous gradient (\emph{strongly
    smooth}).
In order to prove Theorem \ref{thm:rank1_rc}, we show that
with high probability the function $f$ satisfies 
the local curvature condition, which is analogous to strong
convexity, and the local smoothness condition, which is analogous to strong
smoothness. 

\begin{enumerate}
    \item[\textbf{C1}] \quad \emph{Local Curvature Condition} 

        \vspace{0.3cm}
        There exists a constant $C_1$ such that for any $Z$ satisfying $d(Z, \Zstar)
        \leq \sqrt{\frac{3}{16}\sigma_\r} $, 
    \[ \ip{\nabla f(Z), Z - \Zbar } \geq C_1 \norm{Z - \Zbar}^2_F + 
        \norm{(Z - \Zbar)^\T \Zbar}^2_F. \]

    \item[\textbf{C2}] \quad \emph{Local Smoothness Condition} 

        \vspace{0.3cm}
    There exist constants $C_2, C_3$ such that  for any $Z$ satisfying $d(Z,
    \Zstar) \leq \sqrt{\frac{3}{16}\sigma_\r}$,
    \[ \norm{\nabla f(Z)}^2_F \leq C_2 \norm{Z - \Zbar }^2_F + C_3 
        \norm{(Z - \Zbar)^\T \Zbar}^2_F. \]
\end{enumerate}

\subsection{Proof of the Local Curvature Condition}
\[
    \begin{aligned}
        \ip{\nabla f(Z), H} & = && \overbrace{\frac{2}{m} \sum_{i=1}^m \trace(H^\T A_i
        \Zbar)^2 } ^ {p^2}+
        \overbrace{\frac{1}{m} \sum_{i=1}^m \trace(H^\T A_i H)^2 } ^{q^2}+ 
        \frac{3}{m} \sum_{i=1}^m \trace(H^\T A_i \Zbar)\trace(H^\T A_i H) \\
        & \geq && p^2 + q^2 - \frac{3}{m} \sqrt{ \sum_{i=1}^m \trace(H^\T A_i
        \Zbar)^2 } \sqrt{ \sum_{i=1}^m \trace(H^\T A_i H)^2 } \\ 
        & = && p^2 + q^2 - \frac{3}{\sqrt{2}} \overbrace{ \sqrt{ \frac{2}{m}
                \sum_{i=1}^m \trace(H^\T A_i
            \Zbar)^2 } }^ p \overbrace{   \sqrt{\frac{1}{m} \sum_{i=1}^m
            \trace(H^\T A_i H)^2 } } ^ q\\
        & = && \left(p - \frac{3}{2\sqrt{2}}q \right)^2 - \frac{1}{8}q^2 \\
        & \geq && \left(\frac{p^2}{2} - \frac{9}{8}q^2 \right) - \frac{1}{8}q^2 \\
        & = && \frac{p^2}{2} - \frac{5}{4}q^2 = \frac{1}{m} \sum_{i=1}^m
        \trace(H^\T A_i \Zbar)^2  - \frac{5}{4}
        \frac{1}{m} \sum_i \trace(H^\T A_i H)^2 \\
        & \geq && \left(\sigma_\r - \frac{\delta}{2}\right) \norm{H}^2_F + \norm{H^\T \Zbar}^2_F - \frac{5\delta}{4}
        \norm{H}^2_F - \frac{5}{2} \norm{HH^\T}^2_F \\
        & \geq && \left(\sigma_\r -\frac{5}{2} \norm{H}^2_F - \frac{7}{4}\delta
    \right) \norm{H}^2_F + \norm{H^\T \Zbar}^2_F.  
    \end{aligned}
\]
where we use Cauchy-Schwarz inequality in the 2nd line, the inequality $(a-b)^2 \geq \frac{a^2}{2} - b^2$ in
the 5th line, Lemma \ref{lem:rank2_lem_M} and
\ref{lem:rank2_lem_hess} in the 7th line, 
and the fact that $\norm{HH^\T}_F
\leq \norm{H}^2_F$ in the 8th line.
Since $\norm{H}_F \leq \sqrt{\frac{3}{16}\sigma_\r}$ and $\delta \leq
\frac{1}{16}\sigma_\r$, we have
\begin{equation} 
    \label{eq:rank2_curvature} 
    \ip{\nabla f(Z), H} \geq  \frac{27}{64} \sigma_\r  \norm{H}^2_F +
    \norm{H^\T \Zbar}^2_F .
\end{equation} 

\subsection{Proof of the Local Smoothness Condition}
We need to upper bound $\norm{\nabla f(Z)}^2_F = \max_{\norm{W}_F = 1} \abs{ \ip{\nabla f(Z), W} }^2$.
It suffices to show that for any $W \in \R^{n \times R}$ of unit Frobenius norm, $\abs{ \ip{\nabla f(Z), W}}^2 $ is
upper bounded if $Z \in E\left( \sqrt{ \frac{3}{16}\sigma_r}\right)$.

Since $(a+b+c+d)^2 \leq 4(a^2+b^2+c^2+d^2)$, we have
\[
    \begin{aligned}
        \abs{ \ip{\nabla f(Z), W} }^2   & = && \left( \frac{1}{m} \sum_{i=1}^m
        \left( \trace(H^\T A_i H) + 2 \trace(H^\T A_i \Zbar) \right) \left(
    \trace(W^\T A_i H) + \trace(W^\T A_i \Zbar)\right) \right)^2 \\
        & = && \bigg( \frac{1}{m} \sum_{i=1}^m \trace(H^\T A_i H)\trace(W^\T A_i
    H) + 2\trace(H^\T
    A_i \Zbar)\trace(W^\T A_i H) \\
        &   && \hspace{1cm}+ \trace(H^\T A_i H)\trace(W^\T A_i \Zbar) + 2\trace(H^\T A_i
        \Zbar)\trace(W^\T A_i \Zbar) \bigg)^2 \\
        & \leq &&  4 \left(  \frac{1}{m} \sum_{i=1}^m  \trace(H^\T A_i H)\trace(W^\T A_i H) \right)^2 +  4 \left(   \frac{2}{m}
            \sum_{i=1}^m   \trace(H^\T A_i \Zbar)\trace(W^\T A_i H)   \right) ^2 \\
        & && + 4 \left(  \frac{1}{m} \sum_{i=1}^m  \trace(H^\T A_i H)\trace(W^\T A_i
            \Zbar) \right)^2 + 4 \left(  \frac{2}{m} \sum_{i=1}^m  \trace(H^\T A_i
            \Zbar)\trace(W^\T A_i \Zbar) \right)^2. \\
    \end{aligned}  
\]
The first term in the righthand side can be upper bounded as
\[
    \begin{aligned}
        4 \left( \frac{1}{m} \sum_{i=1}^m  \trace(H^\T A_i H)\trace(W^\T A_i H)
        \right)^2  & \leq && 4 \left( \frac{1}{m} \sum_{i=1}^m \trace(H^\T A_i
            H)^2  \right) \left( \frac{1}{m} \sum_{i=1}^m \trace(W^\T A_i H)^2 \right) \\
        & \leq && 4\left(2\norm{H}^4_F + \delta \norm{H}^2_F \right)
            \left(\frac{1}{m} \sum_{i=1}^m \norm{W}^2_F
            \norm{A_iH}^2_F \right) \\ 
        & = && 4\left(2\norm{H}^4_F + \delta \norm{H}^2_F \right)
            \left(\frac{1}{m} \sum_{i=1}^m \norm{A_iH}^2_F \right) \\ 
        & \leq && 4\left(2\norm{H}_F^4 + \delta \norm{H}_F^2 \right) \left(\frac{1}{m} \sum_{i=1}^m \norm{A_i}^2\norm{H}_F^2 \right) \\
        & \leq && 36n \norm{H}_F^2 \left(2\norm{H}_F^4 + \delta \norm{H}_F^2 \right),
    \end{aligned}
\]
where we use the Cauchy-Schwarz inequality in the first and second line, Lemma
\ref{lem:rank2_lem_M} and $\norm{HH^\T}_F \leq \norm{H}^2_F$ in the third line
and Corollary \ref{lem:Ai_operator_norm} in the last line.

The other three terms are bounded similarly. For the second term,  we have
\[
    \begin{aligned}
    4 \left(   \frac{2}{m} \sum_{i=1}^m   \trace(H^\T A_i \Zbar)\trace(W^\T A_i H)   \right) ^2
        & \leq && 16 \left(\frac{1}{m} \sum_{i=1}^m  \trace(H^\T A_i \Zbar)^2 \right) 
        \left(\frac{1}{m} \sum_{i=1}^m \trace(W^\T A_i H)^2\right)\\
        & \leq && 36n \norm{H}_F^2 \left( (4\sigma_1 + 2\delta) \norm{H}_F^2 +
        4\norm{H^\T \Zbar}^2_F \right),
    \end{aligned}
\]
where we use Lemma \ref{lem:rank2_lem_hess} and \ref{lem:Ai_operator_norm}. The
third term is bounded as
\[
\begin{aligned}
    4 \left( \frac{1}{m} \sum_{i=1}^m  \trace(H^\T A_i H)\trace(W^\T A_i \Zbar) \right)^2
    & \leq && 4 \left(\frac{1}{m} \sum_{i=1}^m  \trace(H^\T A_i H)^2 \right) 
    \left(\frac{1}{m} \sum_{i=1}^m \trace(W^\T A_i \Zbar)^2\right)\\
    & \leq && 36 n \norm{\Zbar}_F^2\left(2\norm{H}^4_F +  \delta \norm{H}^2_F
    \right), \\
\end{aligned}
\]
and the fourth term is bounded as
\[
    \begin{aligned}
        4 \left( \frac{2}{m} \sum_{i=1}^m   \trace(H^\T A_i \Zbar)\trace(W^\T A_i
        \Zbar)   \right) ^2
        & \leq && 16 \left(\frac{1}{m} \sum_{i=1}^m \trace(H^\T A_i \Zbar)^2 \right) 
        \left(\frac{1}{m} \sum_{i=1}^m (W^\T A_i \Zbar)^2\right)\\
        & \leq && 36n \norm{\Zbar}^2_F \left( (4\sigma_1 + 2\delta)
        \norm{H}^2_F +
        4\norm{H^\T \Zbar}^2_F \right).
    \end{aligned}
\]
Putting these inequalities together, we have
\[
        \norm{\nabla f(Z)}^2_F 
 \leq 36 n \left(\norm{\Zbar}_F^2 + \norm{H}_F^2\right) \left(2 \norm{H}_F^4
+ (4\sigma_1
+3\delta) \norm{H}_F^2 +  4  \norm{H^\T \Zbar}^2_F \right).
\]
Hence,
\[
        \frac{ \norm{\nabla f(Z)}^2_F }{144 n \left(\norm{\Zbar}^2_F +
            \norm{H}_F^2\right) } 
    \leq 
                \left( \sigma_1 +  \frac{1}{2} \norm{H}^2_F + \frac{3}{4} \delta \right)
    \norm{H}_F^2 + \norm{H^\T \Zbar}^2_F.
\]
Since $\norm{H}_F \leq \sqrt{\frac{3}{16}\sigma_\r}$ and $\delta \leq
\frac{1}{16}\sigma_\r$, we have
\[
    \frac{ \norm{\nabla f(Z)}^2 }{ 144 n \left(\norm{\Zbar}^2_F + (3/16)\sigma_\r \right) } \leq 
    \left( \sigma_1 + \frac{9}{64} \sigma_\r \right) \norm{H}^2_F + \norm{H^\T \Zbar}^2_F.
\] 

\subsection{Proof of the Regularity Condition}
Now we combine the curvature and the smoothness conditions. 
For any $\gamma \in \left(0,\frac{\sigma_1}{\sigma_\r}\right)$, it holds that
\begin{equation} 
    \label{eq:rank2_smoothness}
    \gamma \frac{\sigma_\r}{\sigma_1} \cdot \frac{  \norm{\nabla f(Z)}^2_F }{
    144n \left(\norm{\Zbar}^2_F + (3/16)\sigma_\r \right) } \leq 
    \gamma \frac{\sigma_\r}{\sigma_1} \cdot \left( \sigma_1 + \frac{9}{64} \sigma_\r \right) \norm{H}^2_F +
    \norm{H^\T \Zbar}^2_F. 
\end{equation} 

Combining equation \eqref{eq:rank2_curvature} and \eqref{eq:rank2_smoothness},
we obtain
\[
    \begin{aligned}
        \ip{\nabla f(Z), H} & \geq && \left( \frac{27}{64}-\gamma - \gamma
        \frac{\sigma_\r}{\sigma_1} \frac{9}{64}  \right) \sigma_\r 
    \norm{H}^2_F 
    + \gamma \frac{\sigma_\r}{\sigma_1} \cdot \frac{ \norm{\nabla f(Z)}^2_F
    }{ 144n (\norm{\Zbar}^2_F + (3/16) \sigma_\r) } \\
    & \geq && \left( \frac{27}{64} - \frac{73}{64}\gamma \right) \sigma_\r \norm{H}^2_F
    + \gamma \frac{\sigma_\r}{\sigma_1} \cdot \frac{ \norm{\nabla f(Z)}^2_F
    }{ 144n (\norm{\Zbar}^2_F + (3/16) \sigma_\r) }.
    \end{aligned}
\]
If we take  $\gamma = \frac{1}{3}$, then
\[
    \begin{aligned}
    \ip{\nabla f(Z), H} & \geq && \frac{1}{24}\sigma_\r \norm{H}^2_F
    + \frac{\sigma_\r}{ \sigma_1} \cdot 
    \frac{ \norm{\nabla f(Z)}^2_F }{ 3 \cdot 144n \left(\norm{\Zbar}^2_F + (3/16)
        \sigma_\r\right) }\\
    & \geq && \frac{1}{24} \sigma_\r \norm{H}^2_F + 
      \frac{ \sigma_\r / \sigma_1 }{ 513n \norm{\Zstar}^2_F } \norm{\nabla f(Z)}^2_F,
\end{aligned}
\]
where we use $\norm{\Zbar}^2_F = \norm{\Zstar}^2_F = \norm{X^\star}_* \geq \sigma_\r$.
Thus we have
\[
    \ip{\nabla f(Z), H}  \geq  \frac{1}{\alpha}\sigma_\r \norm{H}^2_F +
    \frac{1}{\beta \norm{\Zstar}^2_F} \norm{\nabla f(Z)}^2_F
\]
for $\alpha \geq 24$ and $\beta \geq \frac{\sigma_1}{\sigma_\r} \cdot 513n$.

\section{Initialization}

\textbf{Proof of Theorem \ref{thm:rank1_init_2}} 

By assumption, we have
\[ \norm{ \frac{1}{m} \sum_{i=1}^m ({\zstar_s}^\T A_i \zstar_s) A_i - 2
        \zstar_s {\zstar_s}^\T} \leq \frac{\delta}{\r}, \quad s \in [\r].    \]
Hence,
\begin{equation} 
    \label{eq:init_spectral_concentration}
 \norm{M - 2X^\star} 
  =   \norm{ \frac{1}{m} \sum_{i=1}^m \sum_{s=1}^\r ({\zstar_s}^\T A_i \zstar_s) A_i - 2 \sum_{s=1}^\r
        \zstar_s {\zstar_s}^T}
    \leq \sum_{s=1}^\r  \norm{ \frac{1}{m} \sum_{i=1}^m ({\zstar_s}^\T A_i \zstar_s) A_i - 2
        \zstar_s {\zstar_s}^\T} \leq \delta.
\end{equation}

Let $\lambda'_1 \geq \cdots \geq \lambda'_n$ be the eigenvalues of $M$.
By Weyl's theorem, we have
\[ 
    \abs{ \lambda'_s - 2 \sigma_s } \leq \delta, \quad s \in [n].
\]
Since $\delta < \sigma_\r$, it is easy to see $\lambda'_1 \geq \cdots
\geq \lambda'_\r > \delta$ and $\abs{\lambda'_s} \leq \delta, s = \r+1, \ldots, n$.
Hence, $\lambda_s = \lambda'_s $, $s \in [\r]$, and $Z^0 {Z^0}^\T$ is the best
rank $r$ approximation of $\frac{1}{2} M$. Therefore,

\[
    \begin{aligned}
        \norm{Z^0 {Z^0}^\T - \Zstar {\Zstar}^\T}_F 
    & \leq && \sqrt{2r} \norm{Z^0 {Z^0}^\T - \Zstar {\Zstar}^\T} \\
    &  =   && \sqrt{2r} \norm{Z^0 {Z^0}^\T - \frac{1}{2}M + \frac{1}{2}M - \Zstar {\Zstar}^\T} \\
    & \leq && \sqrt{2r} \left( \norm{Z^0 {Z^0}^\T - \frac{1}{2}M} + \norm{\frac{1}{2}M - \Zstar {\Zstar}^\T} \right)\\
    & \leq && \sqrt{2r} \delta,
\end{aligned}
\]
where we used $\norm{A}_F \leq \sqrt{\rank(A)} \norm{A}$ in first line, 
the fact $\norm{Z^0 {Z^0}^\T - \frac{1}{2}M} = \frac{1}{2} \abs{\lambda_{r+1}}
\leq \frac{1}{2} \delta$ and inequality
\eqref{eq:init_spectral_concentration} in the last line.

Let $H = Z^0  - \Zbar^0$. We want to bound $d(Z^0, \Zstar)^2 = \norm{H}^2_F$. According to the
discussion in Lemma \ref{lem:rank2_lem_hess_helper3}, $H^\T \Zbar^0$ is
symmetric and ${Z^0}^\T {\Zbar^0}$ is positive semidefinite.

The following step closely follows \cite{TuBocSol15}. It holds that
\[
    \begin{aligned}
        \norm{Z^0 {Z^0}^\T - \Zstar {\Zstar}^\T}^2_F &= &&
        \norm{Z^0 {Z^0}^\T - \Zbar^0 {\Zbar^0}^\T}^2_F\\
        & = && \norm{H {\Zbar^0}^\T + {\Zbar^0} H^\T + H H^\T}^2_F\\
        & = && \trace \bigg ( {\Zbar^0} H^\T H {\Zbar^0}^\T + H {\Zbar^0}^\T H {\Zbar^0}^\T + H
        H^\T {\Zbar^0}^\T \\
        & && \hspace{0.4cm} + {\Zbar^0} H^\T {\Zbar^0} H^\T + H {\Zbar^0}^\T {\Zbar^0} H +  H
        H^\T {\Zbar^0} H^\T \\
        & && \hspace{0.4cm} + {\Zbar^0} H^\T H H^\T + H {\Zbar^0}^\T H H^\T + H H^\T H
        H^\T \bigg ) \\
        & = && \trace \bigg ( (H^\T H)^2 + 2 (H^\T {\Zbar^0})^2 + 2 (H^\T H) ({\Zbar^0}
        ^\T {\Zbar^0}) + 4 (H^\T H) (H^\T {\Zbar^0}) \bigg) \\
        & = && \trace \bigg( \left(H^\T H + \sqrt{2} H^\T \Zbar^0\right)^2 + (4 -
        2\sqrt{2}) (H^\T H) (H^\T \Zbar^0) + 2 (H^\T H) ({\Zbar^0 }^\T \Zbar^0)
        \bigg)\\
        & \geq && \trace \left((4-2\sqrt{2}) (H^\T H) (H^\T \Zbar^0) + 2 (H^\T
            H) (\Zbar^\T \Zbar)      \right) \\
        & = && \trace \left( (4-2\sqrt{2}) (H^\T H) ({Z^0}^\T \Zbar^0)
        \right) + \trace \left(
            (2\sqrt{2} - 2) (H^\T H) (\Zbar^\T \Zbar)  \right), \\
    \end{aligned}
\]
where in the fourth line we used the property that the trace is
invariant under cyclic permutations and $H^\T \Zbar^0 = {\Zbar^0}^\T H$.

Since ${Z^0}^\T \Zbar^0$ is positive semidefinite, $\trace((H^\T H) ({Z^0}^\T
\Zbar^0) )$ is nonnegative. Hence,
\[
    \begin{aligned}
        \norm{Z^0 {Z^0}^\T - \Zstar {\Zstar}^\T}^2_F & \geq &&
        (2\sqrt{2} - 2)\trace\left( (H^\T H) (\Zbar^\T \Zbar)  \right) \\
        & = && ( 2\sqrt{2} - 2) \norm{ H \Zbar^\T}^2_F \\
        & \geq && ( 2\sqrt{2} - 2) \norm{H}^2_F \sigma_\r \\
        & = && ( 2\sqrt{2} - 2)  \sigma_\r d(Z^0, \Zstar)^2. 
    \end{aligned}
\]
If $\delta \leq \frac{\sigma_\r}{4\sqrt{\r}}$, then
\[
    d(Z^0, \Zstar)^2 \leq 
    \frac{ \norm{ Z^0 Z^0 - \Zstar {\Zstar}^\T}^2_F
   }{ (2\sqrt{2} - 2) \sigma_\r }     \leq
    \frac{ 2r \delta^2}{ (2\sqrt{2} - 2) \sigma_\r } \leq
    \frac{3}{16} \sigma_\r.
\]

\section{Sample Complexity}
\label{sec:sample}
In this section, we verify that our assumptions hold with high probability if $m
\geq c n \log n$, where $c$ is a constant that depends on
$\delta$, $\r$, and $\kappa$. Our proof relies on the following
concentration inequality.
\begin{theorem}(Matrix Bernstein Inequality \cite{Tro15})
    \label{thm:matrix_berstein}
    Let $S_1, \ldots, S_m$ be independent random matrices with dimension $n \times n$. Assume that
    $\E(S_i) = 0$ and $\norm{S_i} \leq L$, for all $i \in [m]$. Let $\nu^2 = \max \set{ \norm{
        \sum_{i=1}^m \E(S_i S_i^\T) },  \norm{\sum_{i=1}^m \E(S_i^\T S_i)} }   $. Then
    for all $\delta \geq 0$,
    \[
        \P\left( \norm{ \frac{1}{m} \sum_{i=1}^m S_i } \geq \delta \right) \leq 2n \exp \left(
            \frac{-m^2\delta^2}{\nu^2 + L m \delta / 3   }
        \right).
    \]
\end{theorem}
We first give a technical lemma that we will use later.
\begin{lemma}
\label{lem:GOE_miss}
Let $A = (a_{ij})$ be a random matrix drawn from GOE. Let $S = a_{11} A - 2 e_1
e_1^\T$. There exist absolute constants $C$, $\rho$ such that with probability at least $1 - C e^{-\rho n}$, we have
\[ \norm{S} \leq 18n. \]
\end{lemma}
\begin{proof}
    Let $\tilde{A} = A - a_{11} e_1 e_1^\T$. $S = a_{11} \tilde{A} + (a_{11}^2 -
    2) e_1 e_1^\T$. Note that $a_{11}$ and $\tilde{A}$ are independent, hence
    $ \norm{S} \leq |a_{11}| \| \tilde{A} \| + | a^2_{11} - 2 |$. 
    Besides, since $a_{11} \sim \N(0, 2)$, we can see that $a^2_{11} / 2$ is $\chi^2$ distributed.

First we bound the operator norm of $\tilde{A}$. We rewrite $\| \tilde{A} \|$ as
\[
    \| \tilde{A} \| = \max_{\norm{u}=1} | u^\T \tilde{A} u |  = \max_{ \norm{u}
    = 1 } | u^\T D u - d  u^2_1 | \leq \norm{D} + | d |,
\] 
where $D = \tilde{A} + d e_1 e_1^\T$, $d \sim \N(0, 2)$. As $D$ is GOE
distributed, by Lemma \ref{lem:GOE_operator_norm}, 
\begin{align}
\label{ineq:GOE_norm}
\P\left( \norm{D} > 3\sqrt{n} \right) \leq C' e^{-\rho' n},
\end{align}
where $C'$ and $\rho'$ are absolute constants.

Using the Gaussian tail inequality, we have
\begin{align}
\label{ineq:gaussian}
\P\left( | d | > 2\sqrt{n} \right) \leq 2e^{-n}.
\end{align}
Combining inequalities \eqref{ineq:GOE_norm} and \eqref{ineq:gaussian}, we have
\begin{align}
\label{ineq:GOE_miss}
\P\left( \| \tilde{A} \|  > 5\sqrt{n} \right) \leq \P\left( \norm{D}
> 3\sqrt{n}\lor | d | > 2\sqrt{n} \right) \leq C'e^{-\rho' n} + 2e^{-n} ,
\end{align}
where the last inequality follows from the union bound.

Next we bound the deviation of the $\chi^2$ term. By the corollary of
 Lemma 1 in \citet{LauMas00}, we have
\begin{align}
\label{ineq:chi2}
    \P(| a^2_{11}   - 2 | > 4(\sqrt{n} + n) ) \leq  2 e^{-n}.
\end{align}
Since $a_{11}$ is identically distributed as $d$, inequality \eqref{ineq:gaussian} holds for
$a_{11}$ as well. Namely,
$\P\left( | a_{11} | > 2\sqrt{n} \right) \leq 2e^{-n}$.
Combining this with inequalities \eqref{ineq:chi2}, \eqref{ineq:GOE_miss},
we have
\[
\P\left( \norm{S} \leq 14 n + 4\sqrt{n} \right) \geq 1 - 6e^{-n} - C'
e^{-\rho'
n}.
\]
Finally, the statement is obtained by choosing proper $C$, $\rho$, and using
$\sqrt{n} \leq n$.
\end{proof}

\subsection{Proof of Theorem \ref{thm:rank1_sample_init}}
\begin{proof}
    It is equivalent to show that for any unit vector $u$,
    with high probability,
    \[
        \norm{\frac{1}{m}\sum_{i=1}^m (u^\T A_i u) A_i - 2 u u^\T} \leq
        \frac{\delta}{\r \sigma_1}. 
    \]
    If $P$ is an orthonormal matrix, then
    \begin{align*}
         \norm{\frac{1}{m}\sum_{i=1}^m \left( (Pu)^\T A_i (Pu) \right) A_i -
            2(Pu)(Pu)^\T}
         = &\norm{\frac{1}{m}\sum_{i=1}^m  \left( u^\T (P^\T A_i P)
                u  A_i \right)  - 2Puu^\T P^\T}\\
         = &\norm{\frac{1}{m}\sum_{i=1}^m  u^\T (P^\T A_i P) u  P^\T A_i P  - 2uu^\T} \\
         = &\norm{\frac{1}{m}\sum_{i=1}^m  u^\T \wt{A}_i  u  \wt{A}_i  - 2uu^\T}, 
    \end{align*}
    where in the second line we use unitary invariance of the
    operator norm, and in the last line we denote $ P^\T A_i P$ by $\wt{A}_i$.
    Since the GOE is invariant under orthogonal conjugation, $\wt{A}_i$ and $A_i$ are identically
    distributed. 
    Hence, it suffices to prove the claim when $u = e_1$,
    i.e.
    \[ \norm{\frac{1}{m}\sum_{i=1}^m a^{(i)}_{11}   A_i - 2 e_1 e^\T_1} \leq
    \delta_0, \]
    where $a^{(i)}_{11}$ is the $(1,1)$ entry of $A_i$ and $\delta_0 =
    \frac{\delta}{\r \sigma_1}$.

    To show this, we apply Theorem
    \ref{thm:matrix_berstein}, where
    $S_i = a^{(i)}_{11} A_i - 2 e_1 e^\T_1$.
    This requires that the operator norm of $S_i$ is bounded, for each $i$.  
    We address this by noticing that with high probability $\norm{S_i} \leq 18n$, $\forall i$.
    To be precise, by Lemma \ref{lem:GOE_miss}
    there exist constants $C, \rho$, such that 
    \[
        \P\left(\norm{S_i} > 18n \right) \leq C e^{-\rho n}, \;\; i = 1,\ldots, m.
    \]
    Taking the union bound over all the $S_i$s leads to
    \begin{equation}
        \label{ineq:Si_bounded}
        \P\left( \max_i \norm{S_i} > 18n \right) \leq m C e^{-\rho n}.
    \end{equation}
    Next, we calculate $\nu^2 = \norm{ \sum_{i=1}^m \E( S^2_i ) } = m \norm{\E(S^2_1)}$.
    Let $A = (a_{ij})$ denote $A_1$, $S$ denote $S_1$. We have $\E(S^2) = \E({ a_{11} }^2 A^2 ) - 4
    e_1 e^\T_1$, and
    \begin{align*}
        \big( a_{11}^2 A^2 \big)_{11} &= a_{11}^4 + \sum_{k=2}^n a_{11}^2 a_{1k}^2,\\
        \big( a_{11}^2 A^2 \big)_{ii} &= a_{11}^2 \left( a_{ii}^2 + \sum_{k\neq i}^n a_{ik}^2
        \right), \;\;\forall i\neq1,\\
        \big( a_{11}^2 A^2 \big)_{ij} &= a_{11}^2 \sum_{k=1}^n a_{ik}a_{jk}, \;\; \forall i
        \neq j.
    \end{align*}
    It is easy to see that $\E(a_{11}^2 A^2) = \text{diag}(2n+10, 2n+2, \ldots, 2n+2)$. Consequently,
    $\nu^2 = (2n+6)m.$

    By Theorem \ref{thm:matrix_berstein}, if
    $m \geq \frac{42}{\min(\delta^2_0, \delta_0)} \cdot n \log n$, then 
    \begin{equation}
    \begin{aligned}
        \label{ineq:m_complexity}
        \P\left( \norm{ \frac{1}{m} \sum_{i=1}^m S_i } \geq \delta_0
        \right) &\leq 2n \exp \left( \frac{-m\delta_0^2}{2n(1+3\delta_0) + 6} \right) \\
        &\leq 2n \exp \left( \frac{-m\delta_0^2}{2n(4+3\delta_0)}\right)\\
        &\leq 2n \exp \left( \frac{-m\delta_0^2}{14n \cdot \max(1, \delta_0)}\right)\\
        &\leq \frac{2}{n^2}.
    \end{aligned}
    \end{equation}
    Combining inequalities \eqref{ineq:Si_bounded} and \eqref{ineq:m_complexity},
    we conclude that
    \[
        \P\left( \norm{ \frac{1}{m} \sum_{i=1}^m a^{(i)}_{11} A_i - 2 e_1 e_1^\T
        } \leq \delta_0 \right) \geq 1 - m C e^{-\rho n} - \frac{2}{n^2}. 
    \]
\end{proof}

\subsection{Proof of Theorem \ref{thm:rank1_sample_hess}}
The formulation of the second order partial derivatives and their expectations is
given in Appendix \ref{sec:ingredients}.

It is easy to see that for any $\Zbar \in \tS$,
$\max_{s \in [\r]} \norm{\zbar_\r} \leq \sqrt{\sigma_1}$. 
Thus it is sufficient to prove that for any two unitary vector $u$ and $y$ with
high probability it holds that
\[
    \norm{\frac{1}{m}\sum_{i=1}^m  2 A_i u y^\T  A_i - 2 u^\T y I - 2 yu^\T}
    \leq \frac{\delta}{\r\sigma_1}.
\]

We can decompose $y$ as $y = \beta u + \beta_\perp u_\perp$ for a certain
unit vector $u_\perp$ that is orthogonal to $u$, where $\beta^2 + \beta^2_\perp
= 1$. Let $\delta_0 = \dfrac{\delta}{2\r \sigma_1}$. It suffices to prove the
following two claims.
\begin{enumerate}[(i)]
    \item For any unitary vector $u$, with high probability 
        \[
            \norm{\frac{1}{m}\sum_{i=1}^m  2 A_i u u^\T  A_i - 2 I - 2 uu^\T}
            \leq \delta_0. 
        \]
    \item For any two orthogonal unit vectors $u$ and $ u_\perp$, with
      high probability
        \[
            \norm{\frac{1}{m}\sum_{i=1}^m  2 A_i u u_\perp^\T  A_i - 2 u_\perp u^\T}
            \leq \delta_0. 
        \]
\end{enumerate}
\subsubsection*{Proof of (i)}
If $P$ is an orthonormal matrix, then 
\begin{align*}
    \norm{\frac{1}{m}\sum_{i=1}^m  2 A_i Pu u^\T P A_i - 2I - 2 Puu^\T P^\T}
    &  = \norm{\frac{1}{m}\sum_{i=1}^m  2 P^\T A_i Pu u^\T P^\T A_i P - 2I - 2 uu^\T}\\
    & = \norm{\frac{1}{m}\sum_{i=1}^m  2 \wt{A}_i u u^\T \wt{A}_i - 2I - 2 uu^\T},
\end{align*}
where $\wt{A}_i$ and $A_i$ have the same distribution. Hence we only need to
prove the case where $u = e_1$:
\[ \norm{ \frac{1}{m}\sum_{i=1}^m 2 v^{(i)} {v^{(i)}}^\T - 2I - 2e_1 e_1^\T }
    \leq \delta_0,  \]
where $v^{(i)} = A_i e_1$ is the first column of $A_i$.

Let $S_i = 2(v^{(i)}{ v^{(i)}}^\T  - I - e_1 e_1^\T)$. To apply Theorem
\ref{thm:matrix_berstein}, we need to show that with high probability
$\norm{S_i}$ is bounded for each $i$ and calculate $\nu^2 =
\norm{\sum_{i=1}^n \E(S_i^2)} = m \norm{\E(S^2_1)}$. 

Let $S, v, A$ denote $S_1$, $v^{(1)}$, and $A^{(1)}$ respectively. It is easy to see that 
\[
    \norm{S} \leq 2 \norm{v}^2 + 4 = 2 ( w + a_{11}^2 )  + 4,
\]
where $w= \sum_{k=2}^n  a^2_{1k}$.
As $a_{11} \sim \N(0, 2)$, $a_{1k} \sim ~ \N(0,1)$ for $k \neq 1$,  we can
see that $a^2_{11} / 2$ and $w$
are $\chi^2$ distributed with degrees of
freedom $1$ and $n-1$, respectively.
Using the $\chi^2$ tail bound, we have 
\begin{align*}
    &\P\left( a^2_{11}/2 > 2(\sqrt{n} + n) + 1 \right)  \leq  e^{-n}, \\
    &\P\left( w > 5n-1 \right)  \leq  e^{-n},  \;\; k = 2, \ldots, n.
\end{align*}
It follows from the union bound that
\[ \P\left( \norm{S} > 26n + 6 \right) \leq 2 e^{-n}, \]
and consequently
\begin{equation}
    \label{ineq:Si_bounded_hess}
    \P\left( \max_{i} \norm{S_i} > 26n+ 6 \right) \leq  2me^{-n}.
\end{equation}
To calculate $\nu^2$, we expand $\E(S^2)$ as
\begin{align*}
    \E(S^2) &= 4\E\left((v v^\T )^2\right) - 4(I+ e_1 e^\T_1)^2\\
    &= 4\E\left(\norm{v}^2 v{v}^\T \right) - 4(I+3e_1 e_1^\T).
\end{align*}
Some simple calculations show that
\begin{align*}
    & \left(\norm{v}^2 v{v}^\T \right)_{11} = {v_1}^4 + \sum_{k=2}^n
    {v_k}^2 {v_1}^2,\\
    & \left(\norm{v}^2 v{v}^\T \right)_{jj} = 
    {v_1}^2{v_j}^2 + {v_j}^4 + \sum_{k \neq 1, j}
    {v_k}^2 {v_j}^2, \;\; j = 2, \ldots, n,\\
    & \left(\norm{v}^2 v{v}^\T \right)_{jl} = \sum_{k=1}^n
    {v_k}^2 {v_j} v_l, \;\; j < l.
\end{align*}
As $v_1 \sim \N(0, 2)$, $v_j \sim \N(0, 1)$ for $j\neq 1$,
\begin{align*}
    & \E \left(\norm{v}^2 v{v}^\T \right)_{11}  = 2n + 10, \\
    & \E \left(\norm{v}^2 v{v}^\T \right)_{jj} = n+3,
    \;\; j = 2, \ldots, n, \\
    & \E \left(\norm{v}^2 v{v}^\T \right)_{jl} = 0, \;\; j < l.
\end{align*}
Hence, $\E(S^2) = \text{diag}(8n+24, 4n+8, \ldots, 4n+8)$ and thus
$\nu^2 = m(8n+24)$.

If $m \geq (128/\min(\delta^2_0, \delta_0) ) n \log n$, then by applying Theorem
\ref{thm:matrix_berstein} we can see

\begin{equation}
    \begin{aligned}
        \label{ineq:hess_complexity}
        \P\left( \norm{ \frac{1}{m} \sum_{i=1}^m  2 v^{(i)} {v^{(i)}}^\T -  2I - 2 e_1 e_1^\T
            } > \delta_0 \right) & \leq 2n \exp \left( \frac{ -m\delta^2_0  }{8n +
            24 + (\frac{26}{3}n+2)\delta_0 }
        \right)\\
        & \leq 2n \exp \left( \frac{-m\delta^2_0}{ (128/3) n \max(1, \delta_0)} \right)\\
        & \leq \frac{2}{n^2}.\\
    \end{aligned}
\end{equation}
Combining inequalities \eqref{ineq:hess_complexity} and \eqref{ineq:Si_bounded_hess}
leads to
\[ \P\left( \norm{ \frac{1}{m} \sum_{i=1}^m  2 v^{(i)} {v^{(i)}}^\T -  2I -
            2 e_1 e_1^\T} \leq \delta_0 \right) \geq 1 - 2me^{-n} - \frac{2}{n^2}. \]

\subsubsection*{Proof of (ii)}
We only need to prove the case where $u = e_1$ and $u_\perp = e_2$ due to the same reason
above. That is, 
\[
    \norm{\frac{1}{m}\sum_{i=1}^m 2 v^{(i)} {q^{(i)}}^\T - 2 e_2 e_1^\T } \leq
    \delta_0,
\]
where $v^{(i)}$ and $q^{(i)}$ are the first and second columns of $A_i$.

As before, let $S_i = 2(v^{(i)}{ q^{(i)}}^\T  - e_2 e_1^\T)$
and let $S, v, q, A$ denote $S_1$, $v^{(1)}$, $q^{(1)}$ and $A^{(1)}$ respectively. 
From the proof of (i), we can see that with probability at least $1 - 4e^{-n}$
both $\norm{v}$ and $\norm{q}$ are no larger than $\sqrt{13n + 1}$. Since 
$\norm{ S } \leq 2 \norm{v} \norm{q} + 2$, we have
\[
    \P\left(\max_{i} \norm{S_i} \geq 26n + 4 \right) \leq 4me^{-n}.
\]
Next, we calculate $\nu^2 = m \max \set{ \norm{\E(SS^\T)}, \norm{\E(S^\T S)} } $.
\[ \E(SS^\T) = 4 \E(\norm{q}^2) \E(vv^\T) + 4 e_2 e_2^\T.  \]
\[ \E(S^\T S) = 4 \E(\norm{v}^2) \E(qq^\T) + 4 e_1 e_1^\T.  \]
Some simple calculation shows that $\E(\norm{v}^2) = \E(\norm{q}^2) = n+1$,
 $\E(vv^\T) = I + e_1 e_1^\T$ and $\E(q q^\T) = I + e_2 e_2^\T$.
Hence, 
\[ \E(SS^\T) = 4(n+1)I + 4(n+1)e_1e_1^\T + 4 e_2 e_2^\T, \]
\[ \E(S^\T S) = 4(n+1)I + 4(n+1)e_2e_2^\T + 4 e_1 e_1^\T, \]
and $\nu^2 = 8(n+1)m$.
If $m\geq \frac{78}{\min(\delta_0^2, \delta_0) } n \log n$, then by applying Theorem
\ref{thm:matrix_berstein} we have
\begin{equation}
    \begin{aligned}
        \label{ineq:hess_complexity_ii}
        \P\left( \norm{ \frac{1}{m} \sum_{i=1}^m  2 v^{(i)} {q^{(i)}}^\T -  2
                e_1 e_2^\T
            } > \delta_0 \right) & \leq 2n \exp \left( \frac{ -m\delta^2_0  }{8n +
            8 + (\frac{26n + 4}{3} )\delta_0 }
        \right)\\
        & \leq 2n \exp \left( \frac{-m\delta^2_0}{ 26n \max(1, \delta_0)} \right)\\
        & \leq \frac{2}{n^2}.\\
    \end{aligned}
\end{equation}
This means,
\[ \P\left( \norm{ \frac{1}{m} \sum_{i=1}^m  2 v^{(i)} {q^{(i)}}^\T - 
            2 e_1 e_2^\T} \leq \delta_0 \right) \geq 1 - 4me^{-n} - \frac{2}{n^2}. \]

\section{ADMM for Nuclear Norm Minimization}
\label{sec:admm}
We reformulate the nuclear norm minimizing problem as
\begin{equation}
    \label{eq:trace_primal}
        \min_{X \in \R^{n \times n} } \quad \frac{1}{2\lambda} \norm{\A(X) -
            b}^2 + \norm{X}_*, \\
\end{equation}
where $\lambda > 0$ is the regularization parameter. $\lambda \rightarrow 0$
will enforce the minimizer $X^*_\text{nuc}$ satisfying the affine constraint
$\A(X^*_\text{nuc}) = b$.

We apply ADMM to the dual problem of \eqref{eq:trace_primal}:
\begin{equation}
    \label{eq:trace_dual}
    \begin{aligned}
        \min_{\alpha \in \R^m, V \in \R^{n \times n}} & \quad \frac{\lambda}{2} \norm{\alpha}^2 - \alpha^\T b \\
        \text{subject to} & \quad \norm{V} \leq 1\\
        & \quad \A^\T(\alpha) = V,
    \end{aligned}
\end{equation}
where we introduce an auxiliary variable $V$ to make this problem equality
constrained. 

The augmented Lagrangian of problem \eqref{eq:trace_dual} can be written as
\[
L_\eta (\alpha, X) = \frac{\lambda}{2} \norm{\alpha}^2 - \alpha^\T b +
\mathbf{1}_{\norm{\cdot}\leq 1} (V) + \ip{X,
    \A^\T(\alpha) - V} + \frac{\eta}{2} \norm{\A^\T(\alpha) - V}^2_F,
\]
where $X$ is the multiplier, $\eta$ is the penalty parameter, and
$\mathbf{1}_{\norm{\cdot}\leq 1}$ is the indicator function of the unit spectral
norm ball i.e. $\mathbf{1}_{\norm{\cdot}\leq 1} (V)$ equals $0$ if $\norm{V} \leq
1$ and $+\infty$ otherwise.

Let $\vec(\cdot)$  denote the vectorization of a matrix, whose inverse mapping is denoted by $\text{mat}(\cdot)$. 
We can rewrite the transformations as $\A(X) = \mA \vec(X)$ and $\A^\T(\alpha) = \text{mat}(\mA^\T
\alpha) = \sum_{i=1}^m \alpha_i A_i$,  where $\mA$ is a $m \times n^2$
matrix whose $i$th row is $\vec(A_i)^\T$.

The ADMM starts from initialization $(\alpha^0, V^0, X^0)$ and updates the three
variables alternately. The updates can be computed in close forms:
\[
    \begin{aligned}
    \alpha^{k+1} &= (\lambda I + \eta \mA\mA^\T )^{-1} \bigg(b + \mA \vec\big( \eta V^k -
    X^k\big) \bigg),\\
    V^{k+1} &= \text{proj}\bigg(\sum_{i=1}^m \alpha^{k+1}_i A_i + X^k / \eta
    \bigg),\\
    X^{k+1} & = X^k + \eta \bigg( \sum_{i=1}^m \alpha^{k+1}_i A_i - V^{k+1}
    \bigg),
    \end{aligned}
\]
where $\text{proj}(\cdot)$ is the projection onto the unit spectral norm ball.
Let $X = U\Sigma V^\T$ be the singular value decomposition of $X$,
\[ \text{proj}(X) = U \min(\Sigma, 1) V^\T.   \]
In fact, the update of $V$ can be combined with other steps without being computed
explicitly. One only has to iterate the following two steps:
\[
    \begin{aligned}
    \alpha^{k+1} &= (\lambda I + \eta \mA\mA^\T )^{-1} \bigg(b + \mA \vec \big(
    \eta \sum_{i=1} \alpha^k_i A_i + X^{k-1} - 
    2X^k\big) \bigg),\\
    X^{k+1} & = \text{prox}_\eta \bigg( \eta \sum_{i=1}^m \alpha^{k+1}_i
    A_i + X^k\bigg),
    \end{aligned}
\]
where $\text{prox}_\eta(\cdot)$ is the singular value soft-thresholding
operator defined as 
\[ \text{prox}_\eta(X) = U\max(\Sigma - \eta, 0)V^\T.   \]
The sequence of multipliers $\set{X^k}$ converges to the primal solution of
\eqref{eq:trace_primal}.
To speed up the update of $\alpha$, the Cholesky decomposition of $\lambda I + \eta
\mA\mA^\T$ is precomputed in our implementation.

\end{document}